\documentclass[12pt,a4paper]{article}
\usepackage{amsmath,amssymb,amsthm}
\usepackage{enumerate}
\usepackage[shortlabels]{enumitem}
\usepackage[english]{babel}
\usepackage[utf8]{inputenc}
\usepackage{epigraph}
\usepackage{color}
\usepackage{graphicx}        
\usepackage{xcolor}
\usepackage{tikz-cd}
\usepackage{subcaption}
\usepackage{time}
\usepackage{amsthm}
\usepackage{geometry}
\usepackage{stmaryrd}
\usepackage{hyperref}
\usepackage{floatrow}
\usepackage{enumitem}
\usepackage{authblk}
\usepackage{nomencl}
\makenomenclature

\newtheorem{rmk}{Remark}
\newtheorem*{rmk*}{Remark}
\newtheorem{defn}{Definition}
\newtheorem{prop}{Proposition}
\newtheorem{expres}{Experimental Result}

\def\cA{{\cal A}}

\newcommand{\diag}{\mbox{diag}}

\renewcommand{\epsilon}{\varepsilon}
\renewcommand{\phi}{\varphi}

\def\boxit#1{\vbox{\hrule\hbox{\vrule\kern3pt\vbox{\kern3pt#1\kern3pt}\kern3pt\vrule}\hrule}}
\everymath{\displaystyle}

\everymath{\displaystyle}

\date{}
\author[1]{Jean-Claude Belfiore}
\author[1,2]{Daniel Bennequin}
\author[1]{Xavier Giraud}
\affil[1]{Huawei Wireless Technology Lab. Paris Research Center}
\affil[2]{University of Paris Diderot, Faculty of Mathematics}

\title{Logical Information Cells I}

\begin{document}
	
	\maketitle \tableofcontents
\newpage
\begin{abstract}
	In this study we explore the spontaneous apparition of visible intelligible reasoning in simple artificial networks, and we
	connect this experimental observation with a notion of semantic information. \\
	\indent We start with the reproduction of a DNN model of natural neurons in monkeys, studied by Neromyliotis and Moschovakis in 2017 and 2018, to explaini how "motor equivalent neurons", coding only for the action of pointing, are supplemented by other neurons for specifying the actor of the action, the eye $E$, the hand $H$, or the eye and the hand together $EH$. There appear inner neurons performing a logical work, making intermediary proposition, for instance $E \vee EH$. Then, we remarked that adding a second hidden layer and choosing a symmetric metric for learning, the activities of the neurons become almost quantized and more informative.
	Using the work of Carnap and Bar-Hillel 1952, we define a measure of the \emph{logical value} for collections of such cells. The logical score growths with the depth of the layer, i.e. the information on the output decision increases, which confirms a kind of bottleneck principle.\\
	\indent Then we study a bit more complex tasks, \emph{a priori} involving predicate logic, using the operators $\forall$ and $\exists$. In these experiments, bars of different lengths and colors ($R$,$G$,$B$,...) are
	presented and the network has to decide if some pairs are disjoint $D$, intersect only $IO$ or are related by inclusions $II$. Also for this task, even for two bars, the logical value increases considerably with the
	total depth: three hidden layers being sufficient and necessary. Amazingly, with less layers the network performs well,
	but it uses other strategies, like Fourier
	analysis;  then a bifurcation occurs with more hidden layers.\\
	\indent Also amazing is the fact that the logical population takes almost no account of the statistics in the data;
	for instance $D$ and $IO$ are the most frequent inputs, but most of the neurons eliminate $D$ or $II$.\\
	\indent With a richer learning, for instance varying the lengths of two bars, and exchanging the long and the small ones, the network
	develops a richer and more inventive logic; for instance it becomes able to treat directly with $IO$.\\
	\indent We compare the logic and the measured weights. This shows, for groups of neurons, a neat correlation between
	the logical score and the size of the weights. It exhibits a form of sparsity between the layers. The most spectacular result concerns
	the triples which can conclude for all conditions:
	when applying their weight matrices to their logical matrix, we recover
	the classification. This shows that weights precisely perform the proofs.
\end{abstract}
\newpage
\section{Towards Reasoning Networks \\ (\textit{Réseaux Raisonnants})}
In this study we explore the spontaneous apparition of reasoning in simple artificial networks, and we
interpret this experimental observation in terms of semantic information. \\

\noindent We prefer to write Réseaux raisonnants than Réseaux pensants, which had resonated better with the Roseau pensant of Blaise Pascal
("L’homme n’est qu’un roseau, le plus faible de la nature; mais c’est un roseau pensant", in his Pensées),
because human thinking cannot be summarized by logical grouping and deduction.
\subsection{Generalities}

\indent The general hypothesis is as follows: when confronted to a supervised (or reinforced) classification task, or to the determination of structures
in the data which can be formulated in logical terms, a simple
deep neural network ($DNN$) develops by itself a coherent way of reasoning, which proceeds by decomposing the decision into
elementary steps, and recomposing them. This is the opposite of the current hypothesis
made today, that DNNs have difficulties with compositional structures, privileging global correlations. Now our hypothesis must be made more precise, and limited. In particular, we confirm here the known fact, that, in order to work efficiently and reasonably, the network requires specially good metrics (loss functions) for the output, convenient non-linearities and a sufficiently large number of layers, as soon as the task becomes
more complex.\\
\indent The kind of reasoning which appears here experimentally, shares several essential properties with human reasoning, inference and deduction;
in particular it relies on attentive preferences for certain characteristics of the data, and a kind of
discretization of the message that are related to them, as words do, for instance.  The cumulative experience here is short in time and uses a large set of data, nothing is done by long time evolution,
natural wisdom or preliminary memories. However, as we will report, remarkable inventions are made by the network itself. It makes us expect
that introducing other aspects, like \emph{a priori} knowledge, or long and short term memories, will considerably increase the capacity of discrimination, imagination and reasoning of these networks.\\

In a theoretical companion paper, \emph{Topos and Stacks of Deep Neural Networks} \cite{belfiore2021topos}, we have presented a notion of \emph{semantic functioning} of DNNs, based on topos theory, stacks geometry and logics. This notion generalizes what we use in the present experimental study. These two studies were conducted in parallel, in order to confront top down principles of Topological Information Theory with a bottom up view of the activity of artificial neurons.\\

In another companion study, \emph{Logico-probalisitic Information}, Belfiore, Bennequin and Giraud \cite{logico-probabilist}, to appear soon, we study the relation between the logical information which appears here and the Shannon probabilistic information, applied to subjective probabilities in Bayesian networks of (logical) variables.
Probabilities are certainly necessary for more complex tasks, for the acquisition of uncertain knowledge, and their flexible use.
Even in the present study of very simple networks,
Bayesian rules had interest, however, for the essential part, logic only was sufficient. This probabilistic study connects the present work with the
today current interpretation of semantic information, as it is defined and used for instance by Bao et al. 2011 \cite{bao-basu}, 2014 \cite{basu-bao} and Xie et al. 2020 \cite{xie}, for semantic communication, generalizing Carnap and Bar-Hillel 1952 \cite{CBH52}. For other related directions, we can cite Barwise and Seligman 1997 \cite{Barwise1997-BARIFT}, Floridi 2004 \cite{Floridi2004}, D'Alfonso 2011 \cite{DAlfonso2011-DALOQS}.\\

\subsection{The experiments}
In this contribution, we consider two sets of experiments with Deep Neural Networks (DNNs). One set to highlight the capacity of the network of doing {\em propositional calculus} and the second one to highlight the ability of the network of doing {\em predicate calculus}. 

\subsubsection{Motor equivalence. Propositional Logic.}

In the first experiment an action is registered as an entry, which is made by several possible actors $a,b,c$ in $A$ and the network must decide whom was the actor. In the first layer, each neuron reacts to a given action according to three functions
$\varphi_a,\varphi_b,\varphi_c$ of the physical parameters of the action, say $s\in S$ (concretely an angle). First, only one hidden layer was introduced; the $DNN$ learned and discriminated. To make apparent a logical functioning here, we examined the activity of the neurons of the hidden layer as a function of the actor and of the parameter.
We observed that some cells, for some of the actors, develop, by learning, an almost saturated answer at a common value, which then, does not depend anymore on $s\in S$. This allows to exclude these actors when the activity departs from this value. This can be seen as a preliminary knowledge about the future decision, which can be deduced logically.
This experiment reproduces the results of a network imagined by two neurophysiologists, E.Neromyliotis and A.K.Moschovakis, in $2017$ and $2018$ \cite{nero-moscho-17,nero-moscho-18} to study the
ambiguity of motor equivalent cells, in the spirit of the famous mirror cells, coding for an action independently of the fact
that it is you or someone else which achieves this action.\\
\indent To explore the logical potential of the network, we added another hidden layer. And in this case we saw an impressive diversity of logical analysis, made by cells which saturated at different values; many cells made Boolean statements in $a,b,c$ for different intervals of spiking activity. However, in general no individual cell contained sufficient information to conclude alone, but collectively were able to conclude in a logical manner. \\
The fact that the answer logically deducible from the cells, is correct or not, is also important, and we always verified if this was the case or not; remarkably this was always the case. Another important aspect is the capacity of generalization of the network; we verified that it was very good and induced a nice adaptation of the logical cells.\\
More details are given in the corresponding section below.

\indent In order to understand this experiment and more general situations, we have introduced a (preliminary) notion of logical information.

\noindent We are working with a collection of input data $\Xi$, and
a neural network $X_w$ which has learned its weights $w\in W$ in order to answer a well posed final question about data (which corresponds, in our example, to a classification).
We look at all sets of activities in a given inner layer $L_k$, conditioned by some intermediate statements $P$ about the data (named questions), the goal being the final classification. A proposition $P$ corresponds to a sub-collection $\Xi_P$ of $\Xi$. The corresponding sets of activities constitute a receptive field. And we suppose that some neurons develop quantized activities $\epsilon$ for some propositions. We introduce then the set $D_k$ made of those $\epsilon$'s and their complements. A definition of the semantic information in $L_k$, with respect to the final questions at the output layer, is the collection $I_k^{\epsilon},\epsilon\in D_k$ of the propositions that can be decided (i.e. proved or disproved) from each collection of activations states $\epsilon \in D_k$.\\
\indent Therefore this information depends both on the known individual receptive fields, attached to semantics propositions, and on specific sets of activations in the whole layer. It is important to remind that we are forced, by the experimental setting, to consider sets of sets of propositions like $\epsilon_a\models P_1,...,P_m$ for $a\in L_k$. The question becomes: can we deduce logically the final questions, the classification, by considering only $L_k$, knowing the statistics of some responses, and using the sets $I_k^{\epsilon(\xi)}$, associated to a possible input data $\xi\in \Xi$? And if not, which part of the classification can be decided? 
Then we interpret the collection of sets $I^{\epsilon}_k, \epsilon\in D_k$, as a model in $L_k$ of the global problem posed to the network. In other terms, each collection $I^{\epsilon}_k$, for a given $\epsilon\in D_k$, is viewed as axioms for a theory, and we ask if the final questions are decidable or not in this theory \cite{LK-book}. 
\begin{rmk*}
	\normalfont Nothing, a priori, forbids to collect information from several layers and compare them, asking what a layer knows about another one. This leads to a notion of shared information in the network\footnote{We interpret these shared information in terms of categories, functors and natural transformations in a Grothendieck topos}.
\end{rmk*}
\begin{rmk*}
	\normalfont The semantic information in a layer does not purely describe objective operations of a network which has learned because the necessary saturation,
	for having logical cells, depends on three almost independent factors: 
	\begin{enumerate}
		\item the data collection $\Xi$,
		\item the network $X$ and the metric used for learning,
		\item our own choice of the set of intermediate semantic propositions $\{P\}$, for generating saturation over the corresponding subsets $\Xi_P$
		\footnote{Theoretically it would be possible to cancel point $3$, by considering all propositions, but practically the number of choices is too large for that. As it happens in Physics, the result of the experiment depends on the theory and on the experimental design, in particular, what is measured. We will present, in appendix, the exact parameters we have used in the experiments.}. 
	\end{enumerate}
\end{rmk*}

\subsubsection{Topology of colored segments. Predicate Logic.}

\indent With the second example we tested the ability of a simple
DNN to manage predicates logics. This example is inspired by image analysis or speech analysis, but it is also
extremely simple, considering  one dimensional images of two or three colored segments, or the superposition of two of three voices in time, and asking if they intersect or not, and if one is included in another, or superpose with it. The main interest was the passage from usual \emph{propositional calculus} to \emph{predicate calculus},
involving quantifiers, existential and universal. Here also everything worked well, at the condition of increasing the number of layers to at least three.\\
One of the amazing inventions of the network that we observed, was the comparison of the lengths of the objects (respectively the sentences) in the absence of any
questions about these lengths: the network understood by itself that the inclusion is possible only in one sense, without forgetting the colors (resp. the timbers). This allowed it
to generalize fairly well when the colors were exchanged. \\

With one or two hidden layers, the logical behavior was obscure. Very interestingly, with two hidden layers a kind of Fourier analysis
is developed by the network. But with three hidden layers, we observed a
wonderful set of quantized logical cells. Importantly, these logical cells were only interested by propositions which are consequences of
the output questions, and together they can answer these questions after two layers. In some sense this tells us that the propositional
calculus coming from the objectives dominates. However, also importantly, the propositions which are more complex than the others
from the point of view of predicative calculus, posed difficult problem to the cells, and were accessible to them only indirectly,
by complementing the direct decisions.\\

Thus we get a kind of dynamics of information from layers to deeper layers. Then we have a version of semantic bottleneck.
In a companion paper \cite{logico-probabilist}, considering the link with probabilistic inference, we discuss the relation of this experimental discovery
with the Bottleneck principle of Tishby, Pereira and Bialek 2000 \cite{bottleneck-2000}, Tishby and Zaslavsly 2015 \cite{bottleneck-2015}.\\
\indent It is important to say that in the two above experiments, the minimum of error, around $1/100$
is achieved with one hidden layer, and maintained with two and three hidden layers. However the logical functioning progresses with the number of added hidden layers,
showing that the semantic information increases with the depth of the network, then the minimization principle induces a maximization of information. It could be that the form of the
back-propagation algorithm, which looks like a belief propagation algorithm, is responsible of this shift to semantic and logic.\\

More complex tasks, for instance the complete description of the topology of three colored segments,
provoke the appearance of probabilistic estimations: the cells behave as Bayesian estimators, the
quantization is not so good, but the collective decisions are good. The understanding of their
information content needs a threshold, but fundamentally the principles are unchanged.
See \emph{Logico-probabilistic Information} \cite{logico-probabilist}.\\

\noindent Note that, during the work which is reported here, we had the impressions of a new kind of Physics, with biological flavors, in interaction with humans problems
and some aspects of human behaviors.

\subsubsection{Measuring Logics and Semantics}
\indent For understanding these experiments and, we hope, also more general situations, we have introduced a schematic notion of logical information.
Remind that we are working with a collection of input data $\Xi$, and
a neural network $X_w$ which has learned its weights $w\in W$ in order to answer a well posed final question about data (which corresponds to a classification).
Then we look at the whole sets of activities in a given inner layer $L_k$, conditioned by some intermediary statements $P$ about the data; we name them questions, in direction of the final classification. A proposition $P$ corresponds to a sub-collection $\Xi_P$ of $\Xi$. The corresponding sets of activities constitute a receptive field. And we suppose that some neurons develop discretized activities $\epsilon$ for some propositions. We introduce the set $D_k$ made by these $\epsilon$ and their complements. Then, by definition, the semantic information in $L_k$, with respect to the final questions, at the output, is the collection $I_k^{\epsilon};\epsilon\in D_k$ of the propositions that can be decided (i.e. proved or disproved) from each collection of activations states $\epsilon \in D_k$.\\
\indent Therefore this information depends both on the known individual receptive fields, attached to semantic propositions, and on specific sets of activations in the whole layer. It is important to remind that we are forced, by the experimental setting, to consider sets of sets of propositions like $\epsilon_a\models P_1,...,P_m$ for $a\in L_k$. The question becomes: can we deduce logically the final questions, the classification, by considering only $L_k$, knowing the statistics of certain responses, and using the sets $I_k^{\epsilon(\xi)}$, associated to a possible input data $\xi\in \Xi$? And if not what part of the classification can be decided?  Then we interpret the collection of sets $I^{\epsilon}_k; \epsilon\in d_k$, as a model in $L_k$ of the global problem posed to the network. In other terms, each collection $I^{\epsilon}_k$, for a given $\epsilon\in D_k$, is viewed as axioms for a theory, and we ask if the final questions are decidable or not in this theory. We also propose numerical measures of logical values.\\
\begin{rmk*}
	\normalfont In the more theoretical study, \emph{Topos and Stacks of DNNs} \cite{belfiore2021topos}, we define a more general notion of semantic information, which allows to compare the theories expressed by several layers about what happens in a given layer.
\end{rmk*}
\begin{rmk*}
	\normalfont The semantic information in a layer do not describes purely objective operations of a network which has learned, because the necessary saturation,	for having logical cells, depends on three almost independent factors: 
	\begin{enumerate}[label=\arabic*)]
		\item the collection of data $\Xi$,
		\item the network $X$ and the metric used for learning,
		\item our own choice of intermediate semantic propositions $P$, for generating saturation over the corresponding subsets $\Xi_P$.
	\end{enumerate}
	Theoretically it would be possible to forget point $3$, by considering all propositions, but practically, the number of choices is too large for that. As it happens in Physics, the result of the experiment depends on the theory and on the experimental design, in particular what is measured.
\end{rmk*}

\subsubsection{Neural network description}
In all experiments described in this text, we have used the fully connected network represented in figure \ref{fig-network}.
\noindent
\begin{figure}[ht]
	\begin{center}
		\includegraphics[scale=0.85]{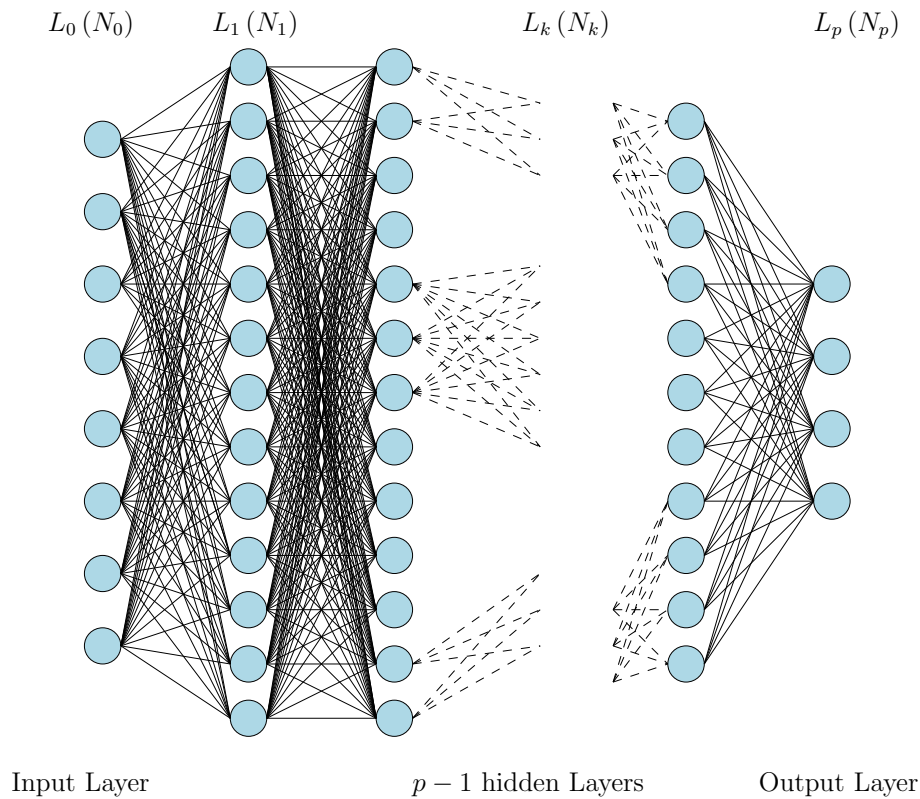}
		\caption{The networks under test}
		\label{fig-network}
	\end{center}
\end{figure}

\noindent We use the following notations:

\begin{center}
	\begin{tabular}{|l|c|}
		\hline
		Object & Notation
		\\\hline
		$k$-th layer & $L_k$\\
		Number of cells in the $k$-layer & $N_k$\\
		\hline
	\end{tabular}
\end{center}

\medskip

Implementations are performed with PYTORCH with the following options:
\begin{itemize}
	\item[$\looparrowright$] Biases are forced to zero.
	\item[$\looparrowright$] Non linear activation functions are the same on each layer, namely $x\mapsto \tanh(ax)$ where $a$ is a positive constant used to improve discretization and to speed up convergence. The chosen values have been selected trough simulation.
	\item[$\looparrowright$] Either Mean Squared Error (MSE) or CrossEntropyLoss (CE) is used as the loss function
	\item[$\looparrowright$] Adam optimizer has been selected.
\end{itemize}
\noindent
All simulations are run on a computer equipped with an intel core i7-8565U CPU.

\section{The simplest model. Propositional calculus.}\label{sec-first-exp}

\subsection{Introduction}

Many neuroscientists have made the observation that the motor system must necessarily involve
cognitive operations, cf. Georgopoulos 2000 \cite{georgopoulos}. Even the simplest animals, like the worm c-elegans or the ascidian
larva, ciona intestinalis, possess a repertoire of voluntary actions, and have to select at the right
time the most convenient one, then select in what order they must execute the sequence of actions, using
memory, anticipation and evaluation. Then it is natural to expect a sort of reasoning in every
animals (in fact every living entity, including plants). The small animals we just mentioned have brains,
containing few hundreds of neurons, interconnected by thousands of synapses, assembled in areas and organized
in moduli, dedicated to several functions. Cf. Kato et  al. 2015 \cite{kato}, Ryan et al. 2018 \cite{ryan2018}.\\
\indent In higher mammals like primates, the brain is much more
complex, but still organized in areas, moduli and networks of sub-systems, and in many cases, the
individual neurons have personal receptive fields, something of interest in the world or in the
functioning. It is not to say that assemblies are not important, to the contrary, they are the
more important ingredient for every perception, memory and decision (cf Hebb's book \cite{hebb-organization-of-behavior-1949}), however these assemblies
rely on the personalities of the individual cells.\\
\indent In their two papers \cite{nero-moscho-17,nero-moscho-18}, E. Neromyliotis and A.K. Moschovakis (N\&M) studied specific neurons
in the pre-motor cortex of monkeys (more precisely in a small region, named arcuate sulcus (AS), and
in periarcuate cortex, both concerned by the movements of the eye and of the fore-limbs. They found
two different sub-populations: 
\begin{enumerate}[label=\arabic*)]
	\item Meq cells (Movement equivalent), which fire during preparation and execution of directed
	movements of the eye and of the arm, without preference for the conditions eye alone (E), hand alone (H) and both eye
	and hand together (EH), but with preference for a goal in space depending of each condition; 
	\item  S-cells which manifest a sort of indifference
	for one or two of the above conditions but continue to prefer some directions, some of them we will call Logical Information Cells, as alluded	in \cite{nero-moscho-17}, because they announce a partial choice of condition.
\end{enumerate}
Taking into account anatomy and timing, the authors suggested that
Meq activity precedes S-cells activity, in order to prepare decision and execution in the primary motor cortex and the spinal chord.\\
\indent N\&M said that all these kinds of cells were already found by Fujii et al. 2002 \cite{fujii}, in other close areas, the supplementary eye field SEF and
the supplementary motor area SMA, specially pre-SMA, the more rostral part of SMA. This region pre-SMA is a crucial node for our discussion, because it
is involved in most of the abstract cognitive processes happening in the brain. For instance Houdé et al. 2000 \cite{houde}, using functional imaging,
have shown that, when shifting from a more perceptual task to a more deductive logical task, there is a shift of brain's activity from a more
posterior network (ventral and dorsal) to a left-prefrontal network, mainly constituted by the middle frontal gyrus MFG, the Broca's area, the anterior insula AI (sic)
and the pre-SMA. For the authors this corresponds to a network supporting logical thinking in general. Further studies have confirmed this view;
for instance Johnston and Leek 2004 \cite{JL-2004}, on mental computations, Tremblay and Small 2010 \cite{tremblay-small}, on language comprehension tasks, either with words either with body gestures.
However, we must mention the interesting  discussion about the necessary link of pre-SMA with a motor action, see Nachev et al. 2008 \cite{nachev2008}, Johnston and Leek 2009 \cite{JL2009}.\\
\indent Of course, thinking and even reasoning, is not limited to pure logical reasoning, for instance the brain conducts probabilistic estimation and inference,
as formalized by Bayes for example (cf Pearl \cite{pearl}, Johnson-Laird et al. 2015 \cite{johnson-laird}), and neuronal networks in the prefrontal cortex well correspond to this
aspect of thinking (Koechlin et al. 2003 \cite{koechlin}). A large network, named the Default Mode Network (DMN), which corresponds to the highly complex activity at rest, is also known to support
spontaneous thinking; it involves several cortical areas, in particular PFC, and sub-cortical regions, like the basal ganglia BG, the thalamus T,
the region around the Hippocampus, and also the Amygdala, known for its expression of all the emotions.
The medial temporal cortex MTL is involved in most of the Long Term Memory operations, in particular episodic and semantic memories, and in MTL
the perirhinal cortex PRC is specially concerned by concepts formations and the understanding of their meaning. Thus the brain uses a network of many networks
for reasoning and performing semantic operations. However, the pre-SMA and its cells surely have a wider role than preparing saccades or
reaching with the arms, in reasoning in general, even if it is hard to separate from some movement operations. This is a good reason for starting with these cells.

\subsection{Experimental settings}
\subsubsection{Input layer description}

The input layer emulate MEQ neuronal responses. It is inspired by biological data though we do not aim at replicating the true biological situation. We have built a layer that is likely to produce
meaningful results.
The input layer is a set of $N_0$ cells corresponding to the MEQ neurons. Given an activator $a\in \mathcal{A}_3 = \{E,H,EH\}$ corresponding to $\mathsf{Eye}$, $\mathsf{Hand}$, $\mathsf{Eye+Hand}$, the neuron $x$ in the input layer gives rise to an activation signal $\phi^{x}_{a}(\theta)\in \mathbb{R}$ defined as
$$\phi^{x}_{a}(\theta):=\frac{\exp\left(\kappa(x,a)\cos(\theta-\mu(x,a))\right)}{2\pi I_0(\kappa(x,a))}$$
$a$ represent the root cause of the signal,
$\mu(x,a)$ is the preferred angle of the neuron $x$ for a given $a\in A_3 = \{E,H,EH\}$ and $\kappa(x,a)$ is related to the inverse of the spread around the mean as shown on the figure below

\begin{figure}[ht]
	\begin{center}
		\includegraphics[scale=0.7]{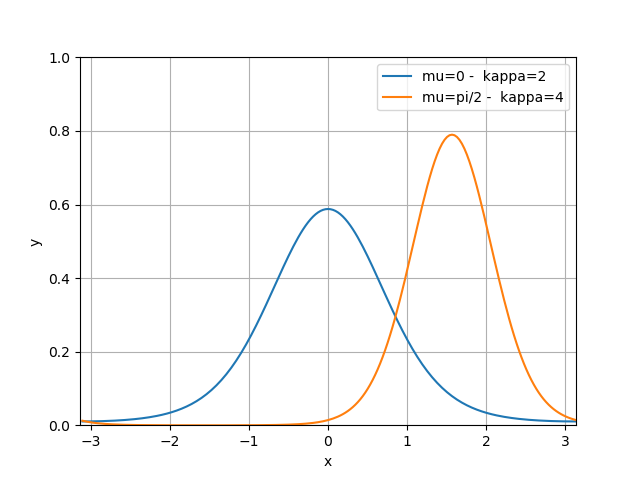}
		\caption{Von Mises distribution}
		\label{fig-VM}
	\end{center}
\end{figure}

\noindent A setting or a batch is a pair $s=(a,\theta)$ and $S_a^{1}$ is the circle $\{(a,\theta),\; \theta\in \mathbb{R}\}$ and the set of the settings is $G_3=S_E^{1}\vee S_H^{1}\vee S_{EH}^{1}$. It gives rise to an activation vector
$$(\phi^{1}_{a}(\theta),\ldots,\phi^{N_0}_{a}(\theta))$$
Additional conditions are required to build the input layer: preferred angle distribution and activation spread can be chosen. The following rules are implemented:

\begin{itemize}
	\item[$\looparrowright$] $]-\pi, \pi]$ is equipartitionned in $N_0$ subintervals with centers $\theta_1, \ldots,\theta_{N_0}$ gathered in a set $C$.
	\item[$\looparrowright$] $\tau_a:x\mapsto \mu(x,a)$ is a permutation of $C$ for all $a\in A_3$.
	
	\item[$\looparrowright$] Relative distributions $x\mapsto \tau_a(x)-\tau_E(x)$ where $a\in\{H,EH\}$ are Gaussian like with well separated maximum. Relative distributions are not that important as long as they are significantly different.
	
	\item[$\looparrowright$] For a given $a\in A_3$, all neurons have the same $\kappa(x,a)$ value, i.e. $x\mapsto\kappa(x,a)$ is constant. We use $\kappa("E")=1.0$, $\kappa("H")=2.0$ and  $\kappa("EH")=1.5$.
	
\end{itemize}

\subsubsection{The network}
\noindent We carry out experiment on three networks described as follows:

\noindent $\bullet\quad$ Number of cells in the $k$-layer

\begin{center}
	\begin{tabular}{|c|c c c c c|}
		\hline
		$p$ & $N_0$ & $N_1$ & $N_2$ & $N_3$  & $N_4$\\
		\hline
		3 & 55 & 50 & 4&&\\
		4 & 55 & 50 & 25 & 4&\\
		5 & 55 & 55 & 50 & 25 & 4\\
		\hline
	\end{tabular}
\end{center}

\noindent $\bullet\quad$ Non linearity: $x\mapsto\tanh(4x)$.

\subsubsection{The output layer and the loss function}
The three activators $a\in\cA_3$ are represented by the three roots of unity $1, \omega, \omega^2$ in order to preserve symetry. The complex number corresponding to $a$ is denoted $z_a$.
Let us assume that a setting $s=(a_0,\theta_0)$ has been selected where $a_0\in\cA_3$ and $\theta_0	\in]-\pi, \pi[$.
The output layer has four neurones:
\begin{itemize}
	\item[-] the first pair provides a complex number $z=x+iy$ and the decision is made towards the activator $a$ minimizing $|z-z_a|$. 
	\item[-] the second pair identifies $\theta$  by means of a pair $(u,v)$ which provides an estimate of $\cos(\theta)$ and $\sin(\theta)$. 
\end{itemize}

\noindent We denote $w$ the set of all weights in the neural network and $f_w$ the map applying a setting $s=(a,\theta)$ on the network output $(x, y, u, v)$.  The set of all weights $w$ minimizes the euclidean distance (MSE criterion)
$$d^2(s, f_w(s))=|z-z_a|^2 + \left(u-\cos(\theta)\right)^2 + \left(v-\sin(\theta)\right)^2$$ 
where $z$ and $z_a$ are defined in the previous subsection.

\subsubsection{Displaying the activity of a neuron}

\noindent Given a setting $s=(a,\theta)$ where $a\in \mathcal{A}_3$ and $\theta	\in]-\pi, \pi[$, we denote the output of the last hidden layer as $\phi^{y}_{a}(\theta)$ where
$y\in [1,N_{p-1}]$.  In order to visualize the response of a neuron for all settings $s=(a,\theta)$, we have associated the discrete valued parameter $a$ with a color and we have
represented the excursion of $\theta\mapsto\phi^{y}_{a}(\theta)$ by means of the 2D polar curve $\mathcal{C}_j$ which plots
$$\theta\mapsto \phi^{y}_{a}(\theta)(\cos(\theta), \sin(\theta))$$
using the color corresponding to $a$. Continuous lines correspond to positive values of $\phi^{y}_{a}(\theta)$ while dashed lines correspond to negative values.

\medskip\noindent From the examples of figure \ref{fig-cells1}, we can observe that when $a\in\{E, EH\}$, $\theta\mapsto\phi^{25}_{a}$ is negative and almost constant. In a similar way, we can observe that when $a\in\{H, EH\}$, $\theta\mapsto\phi^{28}_{a}$ is positive and almost constant.

\begin{figure}[ht]
	\begin{center}
		\includegraphics[scale=0.6]{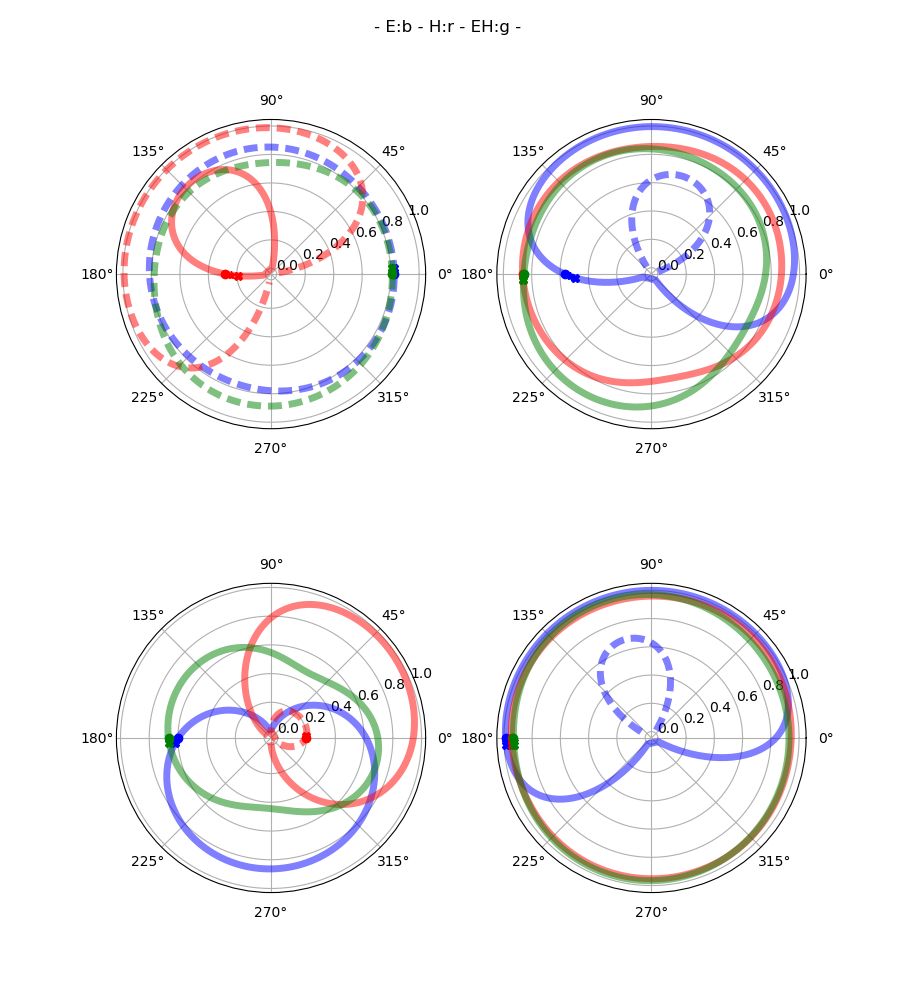}
	\end{center}
\caption{First examples of neuron activities}
		\label{fig-cells1}
\end{figure}


This follows the model explained in \cite{nero-moscho-17,nero-moscho-18}:
\begin{enumerate}[label=\arabic*)]
	\item a first layer $L_1$ contains $55$ neurons of type $\mathsf{Meq}$;
	each one, say $x$, is represented by three $2\pi$-periodic functions $\varphi^{x}_a$ for $a=E,H$ or $EH$, with values in $[0,1]$.
	The value $\varphi^{x}_a(\theta)$ for $\theta \in [0,2\pi]$ represents the activity of the neuron $x$ when the movement is
	made in direction $\theta$ and for the condition $a$;
	\item a hidden layer $L_2$ is made by $50$ neurons with activity in $[-1,1]$, computed
	by a $L-NL$ transformation of the vector measuring the activity in $L_1$:
	\begin{equation}\label{lnlw}
		\psi^{y}=\tanh(\Sigma_x w^{y}_x\varphi^{x}).
	\end{equation}
	\item  a third and last layer is made by four neurons, $z_1, z_2, z_3, z_4$, with activity in $[-1,1]$, the two first ones correspond to
	the condition, the two other ones correspond to the angle.
\end{enumerate}

The coordinates $z_1,z_2$ of the conditions $E$, $H$, $EH$ respectively
correspond to the three vertices of an equilateral triangle in the square $[-1,1]^{2}$: $1=(1,0)$, $\omega=(-1/2,\sqrt{3}/2)$, $\omega^{2}=(-1/2,-\sqrt{3}/2)$. The coordinates $z_3,z_4$ correspond respectively to the cosine and sine of the angle $\theta$. In the functioning feed-forward network they are computed by the fully connected equation \ref{lnlw} from the activity in $L_2$.

The correspondence to be learned by the $NN$ is the natural one: in $L_3$ the description of an individual movement by $(a,\cos \theta, \sin \theta)$,
in entry $L_1$ the corresponding vector $X_1(a,\theta)$.\\

\begin{rmk}\label{rmk-encoding}
	\normalfont The choice of the four cells in $L_3$ is made for respecting at most the symmetries of the experiment.
	We have also tested a model which doesn't respect the symmetry between the three conditions $a$, replacing $z_1,z_2$ by only one neuron $z'_1$,
	taking its values in $[-1,1]$, with $-1$ for $E$, $0$ for $H$ and $1$ for $EH$. We will compare the results of this model  $z'$
	with the model $z$ (see figure \ref{fig-encoding}) in the following sections.
\end{rmk}
\begin{rmk*}
	\normalfont Each input is an angle $\theta$ and a condition $a$, but the neurons $x$	in the first layer don't register this pair, they react to it according to their receptive field, by taking a unique real value $\varphi_a^x(\theta)$. This is not so far from
	primary sensory reactions of schematized retinal cells to a colored
	flash, $\theta$ being the place in the visual plane where the
	flash appears, $a$ being the color
	($L$ for long, red, $M$ for medium, green, $S$ for short, blue), modulating the	reaction of the cell. At this stage, the two components (place and color) are intermingled, and the network has to detect (extract) the color only. Thus,	even if it was not the original motivation of our experiment, this is not very
	far from the usual exploitation of artificial neural networks. (In the visual system of primate, things are a bit different : one layer after the retina, in	the thalamus, most color neurons have a preference for three algebraic combinations of the pigments, $L +M + S$, $L-M$, $L+M-S$).
\end{rmk*}

\noindent The functions $\varphi^{x}_a$ are Von-Mises distributions densities (see Figure \ref{fig-VM}). The sampling for the $55$ cells is uniform in $\theta$,
and contains four sub-populations,
\begin{enumerate}
	\item similar preferred angle for $E,H$ and $EH$,
	\item orthogonal angles for $E$ and $H$ divided in $2.A$, resp. $2.B$ resp. $2.C$, where the preference of $EH$ is almost the same as $E$, {\em resp.} $H$, {\em resp.} another one.
\end{enumerate}

\noindent For comparison of the feed-forward element $F_w(X_1)$ with the truly expected $X_3$, we take the
Euclidian distance, or the Euclidian distance after dilatation of $z_1,z_2$ (resp. $z'_1$ in the asymmetric model mentioned in remark \ref{rmk-encoding}.

\subsection{Theoretical deduction of the movement from the first layer}

Note that the natural map $\Phi$ is from the elements $X_3$ in the disjoint of three circles $G_3=\vee_a S_a^{1}$ to the elements
$X_1$ in the hypercube $K_1=[0,1]^{55}$, then the image $K_3\subset K_1$ of this map contains the set for training, testing and generalizing together.
The network has to compute an inverse $\Phi^{-1}$ of the map $\Phi$ from $K_3$ to $G_3$.\\
We will write $A_3={E,H,EH}$ for the set of conditions. \\
\indent Of course, when functioning, the result of the feed-forward starting with a point $X_1=\Phi(a,\theta)$ does'nt give exactly a point
in $G_3$, it gives a point in the cube $I_4=[-1,1]^{4}$ (or $I_3$ for $z'$). Experiment show that the error is small: this point in $I_4$ (resp. $I_3$)
is very close from the point $(a,\cos \theta, \sin \theta)$.\\

\begin{prop}
	The map $\Phi$ is injective.
\end{prop}
\begin{proof}
	let $X_3=(a,\theta)$ be given in $G_3$, the $55$ components of vector $\Phi(X_3)$ are the numbers
	$\varphi^{x}_a(\theta), 1\leq x\leq 55$. The form of each function $\varphi^{x}_a$ implies that each of its value determines
	$\theta$ up to the symmetry with respect to the angle $\theta^{x}_a$ giving the maximum of $\varphi^{x}_a$. Consequently, as soon as we consider two neurons
	which have different values of $\theta^{x}_a$, the ambiguity is suppressed. We now turn to the condition $a$, and consider a different
	condition $b$; the last of the four families of neurons, i. e. $2.C$, implies that the two vectors $\varphi^{x}_a(\theta), x\in L_1(2.C)$
	and $\varphi^{x}_b(\theta), x\in L_1(2.C)$ are different.
\end{proof}

This proposition doesn't give a very practical algorithm for computing the inverse. We develop now such an algorithm.\\

A direct observation of the neurons in layer $L_1$ explains why they are able to construct $\Phi^{-1}$, at least within a good
approximation.\\
\indent The main observation is the following one: when an angle $\theta$ is given, the population of neurons generates a correspondence between activation value, say $1$ for $\varphi> 1/2$
versus $0$ for $\varphi< 1/2$ and a subset of $A_3={E,H,EH}$ and its complement in $A_3$. We call such a subset a \emph{simple proposition}. We will
meet more elaborate propositions in section \ref{sec-predicate}.\\
\begin{figure}[ht]
	\begin{center}
		\includegraphics[scale=0.6]{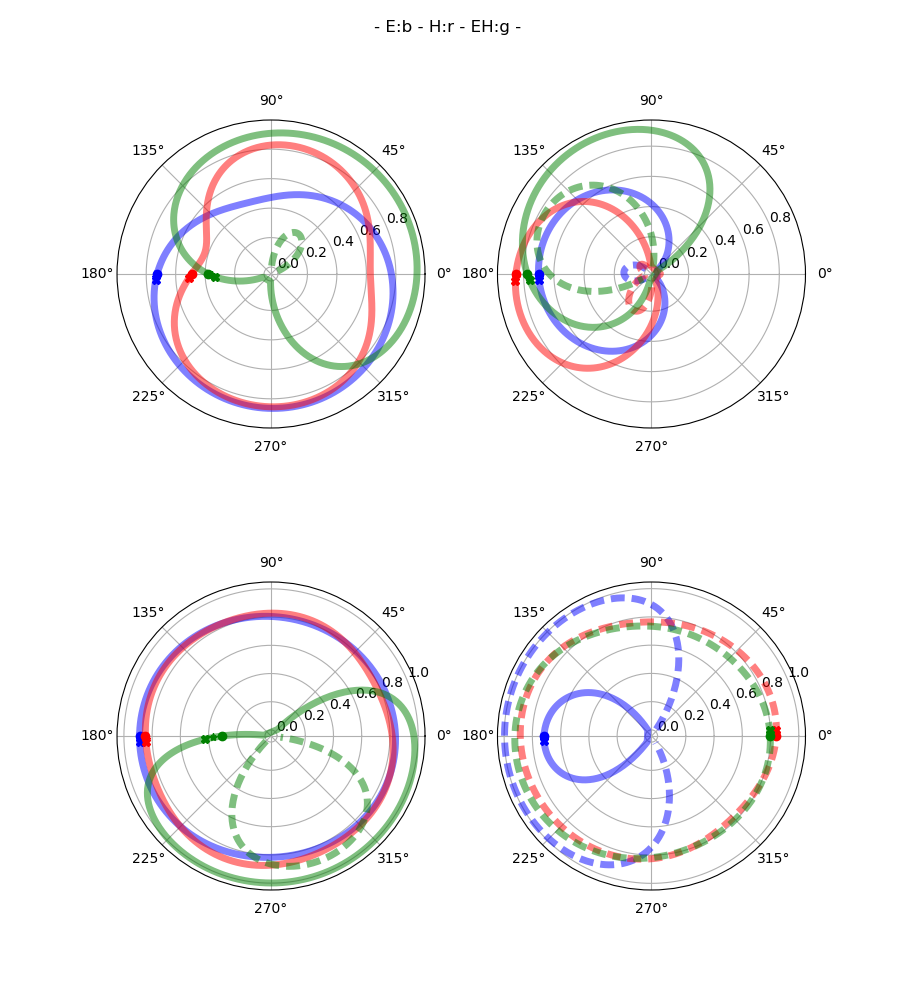}
		\caption{$\theta\mapsto\phi^j_{a}(\cos(\theta), \sin(\theta))$ for $j=17,18,19,20$ and $a\in A_3$}
	\end{center}
\end{figure}
For instance, looking at cell $17$ for $\theta=\pi/2$, we see that the cell is active in conditions $E,H$ but not $EH$, then proposition
$P$ is $E \vee H$, versus its complement $\neg P=EH$.\\
If an ambiguous situation happens, for instance near the value $1/2$ of $\varphi_a^{x}$, we can forget this cell $x$.
However experiment shows that this scarcely happens.\\
For each $\theta$, we check that the set of propositions $P(x)$ is
sufficiently rich to deduce the condition from the vector $\varphi^{x}_a(\theta), x\in L_1(2)$.\\
Note that with the population
of type $1$ only, this would not have happened, the three conditions being non-separated.
But the population $L_1(1)$ is sufficiently rich to determine the angle $\theta$ with a good approximation.\\
These two assertions have to be verified, but they correspond to our choice of distributions of densities $\varphi^{x}_a$.\\

\indent Now a possible logical algorithm works as follows:
\begin{enumerate}[label=\arabic*)]
	\item determine $\theta$ from a particular vector $X_1(17)$ of activity in $L_1(17)$, for instance by linear voting \cite{georgopoulos-1986}: take the sum over $x$
	of the cosine (resp. sine) of its preferred angle (here, in $L_1(17)$, it is the same for the three conditions),
	weighted (i.e. multiplied) by the observed activity $\varphi^{x}$, then take the arccosine (resp. the arcsine).
	\item From this approximate value, deduce the condition, as explained before, by logical computations, either the simplest
	one, either another vote: the number of times $a$ appears in the list of propositions corresponding to the vector $X_1(17)$.
	\item From this condition, use the full population of curves $\varphi^{x}_a$, to get a more accurate value of the angle $\theta$.
	\item Check that this gives the same condition as in step $2$.
\end{enumerate}

This mixture of usual decoding and logics can be seen as a sort of \emph{logical conditioning}, the conditioning being done here on
a continuous parameter like $\theta$.\\

Of course it is not the way this simple DNN has worked. But we will show now that he is probably right, because the hidden
layer contains much more interesting Logical Cells than the first layer, as we will show now.

\subsection{Characteristics of the hidden layer}

In the hidden layer we compute the activity of each cell, a real number between $0$ and $1$, again denoted $\varphi_a^{x}$, corresponding to the
movement directed to $x$ and effected by $a$.\\
\noindent
\begin{figure}[ht]
	\begin{subfigure}{.48\textwidth}
		\centering
		\includegraphics[width=.75\linewidth]{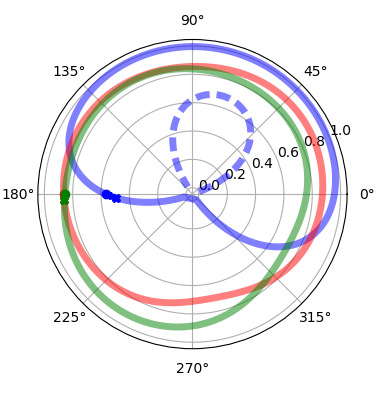}
		\caption{Hidden Layer Cell $26$}
		\label{subf-cell26}
	\end{subfigure}
	\hfill
	\begin{subfigure}{.48\textwidth}
		\centering
		\includegraphics[width=.75\linewidth]{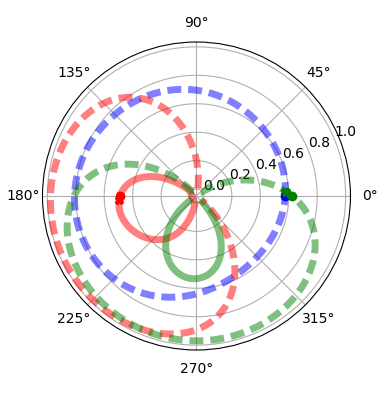}
		\caption{Hidden Layer Cell $30$}
		\label{subf-cell30}
	\end{subfigure}
	\hfill
	\begin{subfigure}{.48\textwidth}
		\centering
		\includegraphics[width=.75\linewidth]{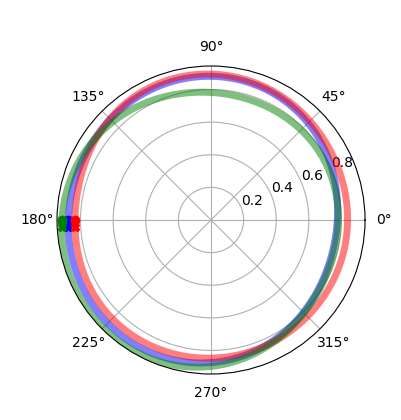}
		\caption{Hidden Layer Cell $11$}
		\label{subf-cell11}
	\end{subfigure}
	\hfill
	\begin{subfigure}{.48\textwidth}
		\centering
		\includegraphics[width=.75\linewidth]{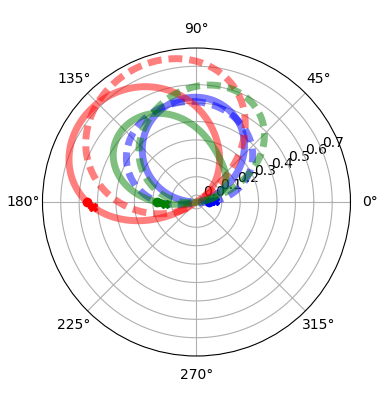}
		\caption{Hidden Layer Cell $16$}
		\label{subf-cell16}
	\end{subfigure}
	\caption{$\theta\mapsto\phi^{j}_{a}(\cos(\theta), \sin(\theta))$ for $j\in \{9,10,11,16\}$}
	\label{fig-hidden}
\end{figure}
The main observation coming from what we see in layer $L_2$, after supervised learning, is that,
in many cells, two of the curves $\varphi_a^{x}$ are almost saturated in $+1$ or $-1$, the same value for both
of them, but the third curve shows positive and negative values. Precisely $35$ cells out of $50$ have this property, that
two of the graphs stay positive or negative, the same for both, and the third one no (cf. Figure \ref{subf-cell26}). In all of these cells, the last graph had exactly two zeros, defining two segments of the circle.\\
\indent Five other cells had only one graph which doesn't change of sign. (Figure \ref{subf-cell30})\\
\indent Five other cells had the three graphs of the same sign. (Figure \ref{subf-cell11})\\
\indent And five cells had the three graphs which changed of signs. (Figure \ref{subf-cell16})\\
The enumeration of the $35$ cells gave $10$ cells of type $E$ (i.e. the transgressing curve is $\varphi_E$), $12$ of type $H$,  and
$13$ cells of type $EH$.\\
The enumeration of the $5$ cells gave four of type $E$, (i.e. the only graph which doesn't change of sign is $\varphi_E$), and one of type $H$.\\
The fundamentally new fact, with respect to $L_1$, is the possibility to use the discretization by the sign, to get logical
propositions. For instance, take one of the $35$, say of type $E$, with the two graphs $\varphi_H$ and $\varphi_{EH}$ positive,
then, if the cell fire negatively, it tells "$E$ is true".\\
For a cell with only one positive graph, say $\varphi_E$, if the cell fire negatively, it tells: $E$ is false, i.e. $H\vee EH$ is true.\\
When the three graphs have the same sign, the cell gives no information at all, its firing seems to be independent of the condition
$E$, $H$ or $EH$.\\
And for another reason, when all the graphs change their sign, the cell gives also no information, because we are not able to extract an information from the activity of the cells, knowing its receptive field.\\

However, from that, the hidden layer is able to recover the condition from its activity and the above logical formulas. Why? This is because, for each condition, the segments in the circle that are informative cover nicely the circle. This was checked by inspection, see Figure \ref{fig-hidden}, for $H$.\\
By definition, a segment is informative for a cell $i$ of type $a$, if it corresponds to the sign which gives the assertion "$a$ is true".\\
Then, assume for instance that the condition is $H$, and the movement at the input is directed to a point $x$, belonging to the informative
segment of a cell $i$ of type $a$, with a sign $+$, then the cell fires positively, consequently the cell tells "$a$ is true".\\
We see no contradiction in the reconstruction, which is not surprising, by construction of the curves $\varphi_a^{x}$, and the preceding definitions.\\

The existence of curves which almost saturate independently of the angle, giving logical propositions, confirmed the observations of Nemyriolitis and Moschovakis \cite{nero-moscho-18}.\\

The appearance of this quantization, and the easy deduction of the classification that it allows, gives us hope that with more layers, more logical functioning could appear.
And this is true.\\

\subsection{More and more layers}

We added between $L_2$ and the output a second hidden layer $L_3$, fully connected with $L_2$, containing $25$ cells.\\
The inputs were the same. The network learned very well by back-propagation, in fact it learned much better and easier than with only one hidden layer.\\
\indent Now the main observation coming from the contemplation of the last hidden layer $L_3$ was the appearance of neurons
where one or two (or three) of the graphs almost saturated at positive or negative value, but now, contrarily to what happened in $L_2$ before,
these values were sometimes different; which changes radically the things as we will see.\\
\noindent
\begin{figure}[ht]
	\begin{subfigure}{.48\textwidth}
		\centering
		\includegraphics[width=.75\linewidth]{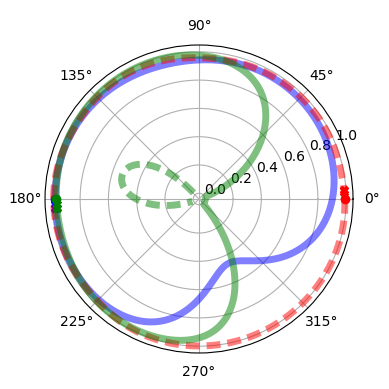}
		\caption{Layer $L_3$ Cell $5$}
		\label{subf-cell-ML-5}
	\end{subfigure}
	\hfill
	\begin{subfigure}{.48\textwidth}
		\centering
		\includegraphics[width=.75\linewidth]{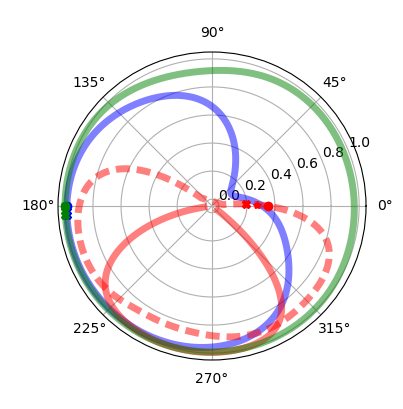}
		\caption{Layer $L_3$ Cell $3$}
		\label{subf-cell-ML-3}
	\end{subfigure}
	\hfill
	\begin{subfigure}{.48\textwidth}
		\centering
		\includegraphics[width=.75\linewidth]{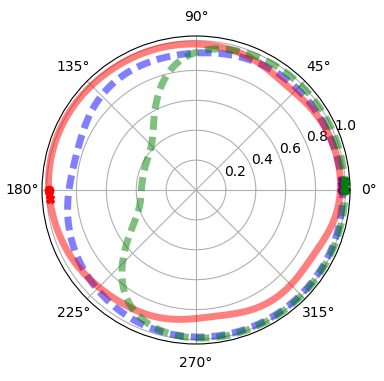}
		\caption{Layer $L_3$ Cell $21$}
		\label{subf-cell-ML-21}
	\end{subfigure}
	\hfill
	\begin{subfigure}{.48\textwidth}
		\centering
		\includegraphics[width=.75\linewidth]{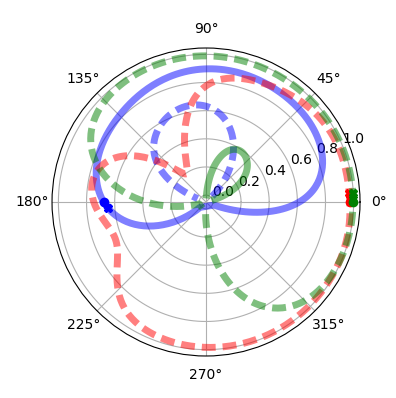}
		\caption{Layer $L_3$ Cell $1$}
		\label{subf-cell-ML-1}
	\end{subfigure}
	\caption{Multi-Layer $L_3$ with $\theta\mapsto\phi^{j}_{a}(\cos(\theta), \sin(\theta))$ for $j\in \{5,3,21,1\}$}
	\label{fig-L3}
\end{figure}

All the (apparently) stupid cells with three times the same sign disappeared. Eight cells had two saturated curves with opposite signs (see Figure \ref{subf-cell-ML-5}).
Seven cells had two saturated curves with the same sign, as in the $L_2$ before (see Figure \ref{subf-cell-ML-3}). One cell had three almost saturated curves, not all of the same sign (see Figure \ref{subf-cell-ML-21}).

Seven cells had only one saturated curve  (see Figure \ref{subf-cell-ML-1}). Eventually,  two cells had behaviors difficult to interpret, crossing all the curves.

Another noticeable thing with respect to the preceding simpler network, was that five out of seven cells with two saturations
of the same sign had an informative segment covering half of the circle. This
allowed these $7$ cells to do almost as well as the preceding  $35$ informative cells, in applying the argument of reconstruction by coverings.
(More precisely, seven of the $25$, versus $35$ cells among the
$50$ in $L_2$ before.)\\

All that gives $23$ logical cells, that are cells whose activity can be translated in a logical proposition.\\
Each of the above type of cells gives a different structure of implication:\\

\noindent \textbf{Examples.} 
($E$ is blue, $H$ is red and $EH$ is green)
\begin{itemize}
	\item Cell 1: $H \Rightarrow -1$, or equivalently $+1 \Rightarrow E\vee EH$. The information is:
	independently of the angle, if the activity is positive, the condition is either $E$ either $EH$. (Figure \ref{subf-cell-ML-1})
	\item Cell 3: $E$ and $EH$ stay strongly positive, but $H$ presents two signs, then negative implies the condition $H$, which we note
	$-1 \Rightarrow H$. The information of this cell is: independently
	of the angle, if the activity is negative, the condition is $H$ (Figure \ref{subf-cell-ML-3}). 
	\item Cell 5: $E \Rightarrow 1$, and $H \Rightarrow -1$, therefore $< 0 \Rightarrow H\vee EH$ and $> 0 \Rightarrow E\vee EH$ (Figure \ref{subf-cell-ML-5}). 
	\item Cell 21: $H\Rightarrow +1$, $E\Rightarrow -1$ and $EH \Rightarrow -1$, then $+1 \Rightarrow H$, $ -1 \Rightarrow E \vee EH$ (Figure \ref{subf-cell-ML-21}). 
\end{itemize}
\begin{rmk*}
	\normalfont In this analysis, as in the following arguments, we use the fact that$P\Rightarrow Q$ is equivalent to $\lnot Q\Rightarrow\lnot P$, and we will use many times that $E,H$ and $EH$ are exclusive. It is legitimate to question these assumptions : how can	the network be aware of the Boolean axioms in logic? The answer is
	that it has learned these elements during the learning process, because the
	asked output is $E,H$ and $EH$ and the metric sanctions any mixture of the
	conditions.
\end{rmk*}

\noindent Let us give a first example of possible fully conclusive reasoning:

\begin{prop}\label{prop:complete}
	Suppose we have three cells of the type of cell 5 of figure \ref{subf-cell-ML-5}, involving symmetrically all possible pairs of arguments, for instance,\\
	cell $I$: $1\Rightarrow E\vee H, -1\Rightarrow E\vee EH$;\\
	cell $II$: $1\Rightarrow E\vee EH, -1\Rightarrow H\vee EH$;\\
	cell $III$: $1\Rightarrow E\vee H,-1\Rightarrow H\vee EH$.\\
	Then the condition follows from the three activities, as soon as they are non-contradictory.
\end{prop}
\begin{proof}
	the following implications are easily verified:
	\begin{multline}
		(1,1,1)\Rightarrow E,\quad (1,1,-1)\Rightarrow \bot, \quad (1,-1,1)\Rightarrow H,\\
		(1,-1,-1)\Rightarrow H,\quad
		(-1,1,1)\Rightarrow E,\quad (-1,1,-1) \Rightarrow EH, \\
		(-1,-1,1) \Rightarrow \bot, \quad (-1,-1,-1)\Rightarrow EH.
	\end{multline}
\end{proof}
\begin{prop}\label{prop:good-enough}
	For an input corresponding to the discrete condition $a$, the above three cells together can
	reconstruct the answer $a$ by pure logical deduction.
\end{prop}
\begin{proof}
	From the symmetry under $\mathfrak{S}_3$ (group of permutations of $3$ objects),
	we can assume that the condition is $E$, then from the table, by contraposition, the cell $II$ fires at $+1$,
	and the cell $III$ also. From that, reading the table in the written direction, we get $\vdash E\vee EH$ and $\vdash E\vee H$,
	then $E$ is asserted to be true.
\end{proof}

\noindent Remarkably, in $L_3$ there existed cells for each of the three types $I,II,III$: $23$ and $25$ are of type $I$,
$5$,$6$,$10$ of type $II$ and $14$, $16$, $18$ of type $III$.\\
Thus the third layer can decide very easily and by pure logical reasoning what is the true condition.\\

\noindent We say that a set of cells which satisfies the result of proposition \ref{prop:complete} is \emph{complete}. And we say that
a set of cells which satisfies proposition \ref{prop:good-enough} is \emph{good enough}. 
\begin{rmk*}
	\normalfont It can happen that a triple has one of these properties without having the other one. Examples are given below.
\end{rmk*}
A set which is both complete and good enough is said to
be \emph{efficient}. \\
To be \emph{conclusive} for a group of logical cells \emph{a priori} depends on its possible activities, and then, on the input it receives.
Efficient always means conclusive for a given set of inputs, and reconstructing the right answer for the required objectives.\\
These  notions are more useful when every set of cells containing a complete (resp. good enough, resp. efficient) set has the
same property. This requirement is equivalent to the absence of contradiction in the propositions coming from real data in input.
A fundamental experimental fact that we observed in this study, is this absence of contradiction
in real data. However a contradiction could \emph{a priori} could happen in "generalization" data, if the new data generate
saturated answers violating the logic, but what we observed in all the unadapted data for generalization,
was more the vanishing of apparent logical structures, for instance no saturation at all.\\

\noindent Other types of cell could be like $21$:\\
cell $IV$: $1\Rightarrow E, -1\Rightarrow H\vee EH$;\\
cell $V$; $1\Rightarrow EH, -1 \Rightarrow E \vee H$.\\

\begin{prop}\label{prop:triple}
	The set of cells $\{I, IV, V\}$ and the set of cells $\{I, II, V\}$ are complete sets, but they are not good enough. However, taken all together, $\{I,II,IV,V\}$ form a set which is good enough, then is efficient.
\end{prop}
\begin{proof}
	For the first triple:
	\begin{multline}
		(1,1,1)\Rightarrow \bot,\quad (1,1,-1)\Rightarrow E, \quad (1,-1,1)\Rightarrow \bot,\\
		(1,-1,-1)\Rightarrow H,\quad
		(-1,1,1)\Rightarrow \bot,\quad (-1,1,-1) \Rightarrow E, \\
		(-1,-1,1) \Rightarrow EH, \quad (-1,-1,-1)\Rightarrow \bot.
	\end{multline}\\
	\noindent For the second one:
	\begin{multline}
		(1,1,1)\Rightarrow EH,\quad (1,1,-1)\Rightarrow E, \quad (1,-1,1)\Rightarrow \bot,\\
		(1,-1,-1)\Rightarrow H,\quad
		(-1,1,1)\Rightarrow EH,\quad (-1,1,-1) \Rightarrow E, \\
		(-1,-1,1) \Rightarrow EH, \quad (-1,-1,-1)\Rightarrow \bot.
	\end{multline}\\
	Now consider an input of type $H$, nobody can tell what will be the predicted activity in $IV$ or $V$. But in $I$, we know that it will be $+1$, and in $II$, we know it will be $-1$, then in the first group we can only conclude $\vdash H\vee E$
	and in the second group $\vdash H\vee EH$.\\
	To verify the efficiency of the union of the groups, we have to look at $E$ and $EH$.\\
	In the case of $E$, cells $II$ and $IV$ express $+1$, then they assert respectively the truth of $E\vee EH$ and $E$,
	which gives $E$.\\
	In the case of $EH$, the cell $I$ tells $-1$, and the cell $V$ tells $+1$, then they respectively conjecture $E\vee EH$ and
	$EH$, thus together they tell that $EH$ is true.
\end{proof}

From the population of receptive fields in the layer $L_3$, we see that the network hesitates between two strategies, one is
a mixture of geography and logic, like the hidden layer $L_2$ before, with conditioning by angular regions, and one purely logical
with efficient groups of cells, that do not look at angles anymore. Of course this could be helpful to develop two possibly cooperating
strategies, but we asked us if a growing complexity will induce a choice or maintain the two ways of reasoning.\\

\indent The result is very instructive: if we increase the number of neurons in $L_3$ (say $50$ instead of $25$), the network regresses to the non-purely logical strategy it adopted with only one hidden layer, but if we increase the number of hidden layers (we introduce a third deeper hidden layer $L_4$ with $25$ neurons), the network totally forgets the primitive
(or initial) strategy, and develops further the logic, it continues using the efficient groups of cells just described above, and it invents new efficient
triples, more directly conclusive, that it was apparently not able to form before, when it had only two hidden layers.\\

With three hidden layers, one $L_2$ with $50$ neurons, and the two deeper ones $L_3$, $L_4$ with $25$ neurons, we found the following innovative composition: $11$ cells of crossed type, like $+1 \rightarrow E\vee H, -1\rightarrow E\vee EH$ (in fact, $1$ as this one, i.e.
like the above cell $I$, $5$ like the cell $II$, preferring $EH$, $5$ like the cell $III$, preferring $H$, which is not totally optimal, but complete and good
enough); $7$ cells of  the fully saturated type (without any stupid one with all saturations on the same sign), and here with optimal distribution ($2$ of the type
$E$ (i.e. $+\pm 1\rightarrow E, \mp 1\rightarrow H\vee EH$), $3$ of the type $H$ (i.e. $\pm 1\rightarrow H, \mp 1\rightarrow E\vee EH$) and $2$
of the type $EH$ (i.e. $\pm 1\rightarrow EH, \mp 1\rightarrow H\vee E$). Six cells were less informative, like $+1\rightarrow E\vee EH$, and a last one
was obscure. (Cf. Figure \ref{fig:ML3}, for a sample in $L_3$.) Remarkably, no cell corresponds to the primitive strategy, mixing geography and logic.\\
\noindent
\begin{figure}[ht]
	\begin{subfigure}{.42\textwidth}
		\centering
		\includegraphics[width=.7\linewidth]{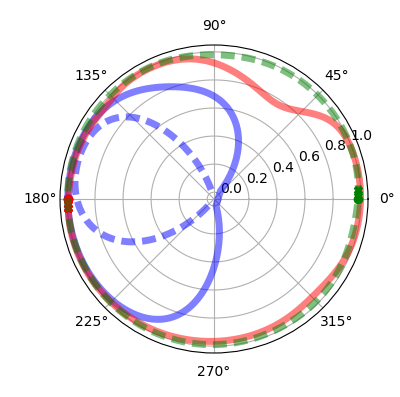}
		\caption{Layer $L_4$ Cell $3$}
		\label{subf-cell-ML3-3}
	\end{subfigure}
	\hfill
	\begin{subfigure}{.42\textwidth}
		\centering
		\includegraphics[width=.7\linewidth]{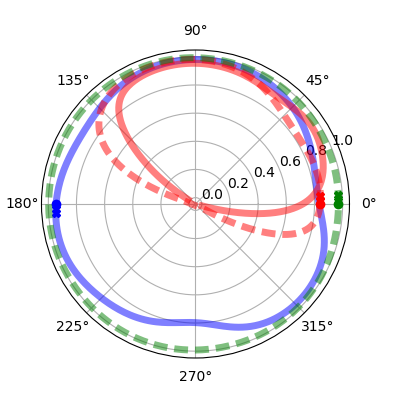}
		\caption{Layer $L_4$ Cell $6$}
		\label{subf-cell-ML3-6}
	\end{subfigure}
	\hfill
	\begin{subfigure}{.42\textwidth}
		\centering
		\includegraphics[width=.7\linewidth]{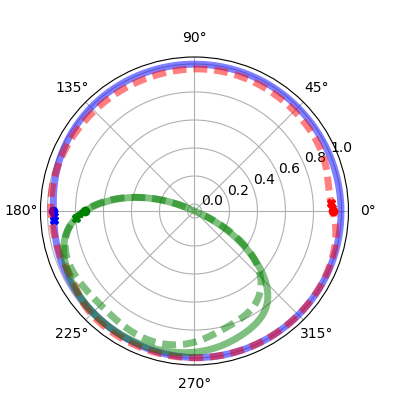}
		\caption{Layer $L_4$ Cell $8$}
		\label{subf-cell-ML3-8}
	\end{subfigure}
	\hfill
	\begin{subfigure}{.42\textwidth}
		\centering
		\includegraphics[width=.7\linewidth]{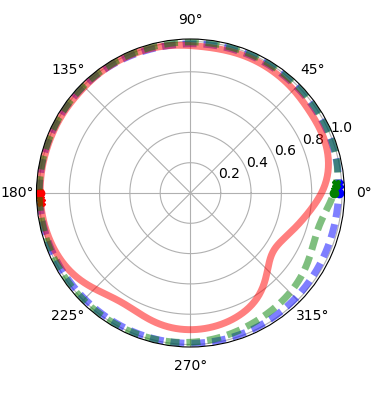}
		\caption{Layer $L_4$ Cell $12$}
		\label{subf-cell-ML3-12}
	\end{subfigure}
	\begin{subfigure}{.42\textwidth}
		\centering
		\includegraphics[width=.7\linewidth]{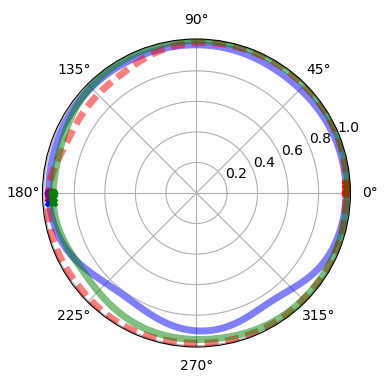}
		\caption{Layer $L_4$ Cell $14$}
		\label{subf-cell-ML3-14}
	\end{subfigure}
	\hfill
	\begin{subfigure}{.42\textwidth}
		\centering
		\includegraphics[width=.7\linewidth]{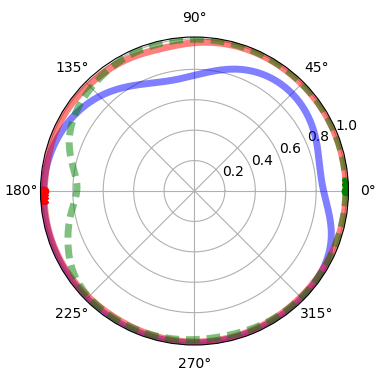}
		\caption{Layer $L_4$ Cell $22$}
		\label{subf-cell-ML3-22}
	\end{subfigure}
	\caption{Multi-Layer $L_4$ with $\theta\mapsto\phi^{j}_{a}(\cos(\theta), \sin(\theta))$ for $j\in \{3,6,8,12,14,22\}$}
	\label{fig:ML3}
\end{figure}
\begin{prop}\label{prop5}
	Consider three maximally saturated cells involving symmetrically all the different pairs of arguments, for instance\\
	cell $VI$: $1\Rightarrow E, -1\Rightarrow H\vee EH$;\\
	cell $VII$: $1\Rightarrow H, -1\Rightarrow E\vee EH$;\\
	cell $VIII$: $1\Rightarrow EH,-1\Rightarrow H\vee E$.\\
	Then the set $\{VI, VII, VIII\}$ is efficient.
\end{prop}

\begin{proof}
	Let us begin by checking that a condition follows from the three activities, as soon as they are non-contradictory:
	\begin{multline}
		(1,1,1)\Rightarrow \bot,\quad (1,1,-1)\Rightarrow \bot, \quad (1,-1,1)\Rightarrow \bot,\\
		(-1,1,1)\Rightarrow \bot,\quad, (1,-1,-1)\Rightarrow E,\quad
		(-1,1,-1) \Rightarrow H, \\
		(-1,-1,1) \Rightarrow EH, \quad (-1,-1,-1)\Rightarrow \bot.
	\end{multline}\\
	Now suppose that the condition is $E$, then $VI$ fires at $+1$, $VII$ at $-1$ (because $E\vee EH\rightarrow -1$ by contraposition) and $VIII$ fires
	at $-1$ (because $E\vee H\rightarrow -1$ by contraposition); then, by the preceding assertion, with $(1,-1,-1)$, the cells together assert that $E$ is true.
\end{proof}

In the preceding layer $L_3$ we had only one cell of this maximally saturated type, then the situation of triples was described by propositions \ref{prop:complete},\ref{prop:good-enough},\ref{prop:triple}.\\
\indent This new possibility, invented in $L_4$, makes $18$ cells out of $25$ all having the appearance of a wise reasoning assembly.\\

\begin{rmk*}
	\normalfont A noticeable difference between the triples in proposition \ref{prop5} and proposition \ref{prop:complete}, is that in \ref{prop5} it establishes a one to one correspondence between the coherent activations (i.e. without contradiction) and the three conditions, but in \ref{prop:complete} it gives a two to one map, i.e. two different activities correspond
	to the same condition.
\end{rmk*}

Now let us look at the effect of doubling the number of cells in $L_3$. The logical functioning collapses; the population becomes even
less logical than in $L_2$ for the network with only one hidden layer. In this $L_2$, $35$ cells over $50$ joined their efforts to reconstruct a
condition by using coverings by informative intervals in a symmetric and uniform manner; here in the new $L_3$, only $17$ cells over
$50$ do that, with seven cells for $E$, seven for $EH$ and three for $H$. Moreover, the majority of the cells, $26$ exactly, are
concerned by the (less informative) unions, five for $E\vee EH$, six for $E\vee H$ and fifteen for $H\vee EH$. The exception which saves the honor (from
the logical point of view), is realized
by three crossed cells (not forming a complete triple, because two repeat the same message), plus two cells having three saturations not all
of the same sign. There exists another good point for this population: it doesn't contain a cell with three saturations
of the same sign, that $L_2$ contained. (But is it a good point to exclude all fantasy?)\\
\indent The explanation of this disaster seems to be related to the well known danger of over-fitting. This is certainly part of the truth, but this is not
all the truth, because the number of weights to learn is
$55\times 50+50 \times 50+4\times 50=50\times 109$, and in the network with three hidden layers the number of weights to learn equals
$55\times 50+50 \times 25+ 25\times 25+25\times 4=50\times 94,5$, which is not far. Of course the three layers imply more non-linearity, but
how to count that? We can just certify:\\
\begin{expres}
	\normalfont With a comparable number of parameters to adapt, the addition of a layer considerably increases
	the logical functioning, at the level of individual cells and of collective behaviors, and the addition of cells in one layer has the
	opposite effect, also at the levels of individual cells and of collective behavior.
\end{expres}
\begin{rmk*}
	\normalfont Please, no deduction about the necessity of a large number of layers in an administration. One can also contest our preference
	for logic, and logical invention, versus fantasy, and leisure, for a comparable result. Do not forget that both networks are successful for the task
	they have to accomplish. We have not yet compared their powers of generalization. The point of view we adopt here, is more our own intelligibility
	of the network's functioning. We do not contest that for some more complex tasks, the addition of many cells in a layer could be preferable,
	or that logic could come later from another road in more complex networks.
\end{rmk*}

\noindent Another remark to temperate the difference: both $L_4$ and the bigger $L_3$ introduced the largest variety of types of logical cells,
even if $L_4$ did that in a much more equilibrate and efficient manner.\\

\indent For the asymmetric model $z'$ (see figure \ref{fig-encoding}), we conducted analog experiments, with one hidden layer of $50$ cells, and two hidden layers
of respectively $50$ and $25$ cells. The main result was the failure to develop logic coherently in both cases. The network
learned well and performed very well the classification, but doesn't develop sufficiently many purely logical cells to be efficient.\\
\noindent
\begin{figure}[ht]
	\begin{center}
		\includegraphics[scale=0.6]{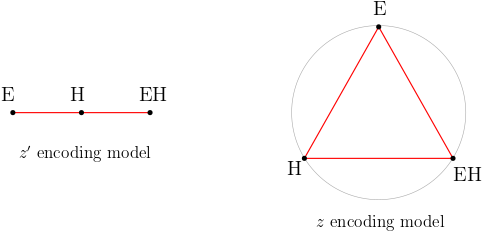}
		\caption{$z$ or $z'$ encoding}
		\label{fig-encoding}
	\end{center}
	\label{fig:encoding}
\end{figure}
\indent With one hidden layer $L_2$, most of the cells have two saturations of the same sign, they are $10$ of type $E$ and $8$
of type $EH$, no one of type $H$. They do the job of the analog cells in $L_2$ for the symmetric model $z$, at least for $E$ and $EH$.
\\
Ten of the cells have one saturation,
two for $E\vee H$, eight for $H\vee EH$, no one for $E\vee EH$. Half of the resting cells have three saturating graphs of the same sign, the other
half develop no saturation.\\
As we see, the condition $H$ encounters difficulties, and it was hard to reconstruct it directly or indirectly from the activity
in the layer in intelligible manner. However, the network is successful with the classification,  including $H$, then this is a case where the functioning is not explained by
what we see in the layers.\\

The symmetry which is respected in the model $z$ and not in the model $z'$ is the group $\mathfrak{S}_3$ of permutations of $A_3$;
the fact that in the model $z'$, $H$ is coded by a point inside the interval $[-1,1]$ and $E$, $EH$ by the boundaries, destroys the symmetry,
and we constat that it also destroys an important
part of the logic inside the network. Cf. Figure \ref{fig-encoding}. \\
\indent Then we get the important conclusion:
\begin{expres}
	\normalfont The topology of encoding must respect the symmetries, for the emergence of a logic
	in the inner layer, but not necessarily for the success of the neural network.
\end{expres}

As we will see in the discussion below, this is reminiscent of the appearance of Fourier analysis or Color analysis in the first hidden layers
of the CNNs, and CNNs are known to be more successful than simple DNNs for image analysis.\\

\indent With two hidden layers, this difficulty persisted. We observed superficially the same kind of progress we saw in the model $z$: it appeared
one graph with three saturations not all of the same sign, and three crossed cells. However, both of them were of the type $\pm 1\rightarrow H\vee E, \mp 1\rightarrow H\vee EH$.
Eight cells had two saturations of the same sign, but no one of the type $H$, and eleven with one saturation, of the type $E\vee H$ or $H\vee EH$, not $E\vee EH$.
Thus
the penalty to $H$ persisted. This implied that it was not possible to logically deduce or reconstruct the condition from the activity. Even taken all together, the
neurons of $L_3$ didn't form a good enough set.\\

\noindent There remain natural questions: 
\begin{itemize}[$-$]
	\item what are the weights of the logical cells for going to the output layer?
	\item Do they reflect the logical preferences?
	\item Same question for significative subsets of the last hidden layer.
	\item Do we see a correlation between synergies of weights and synergies in reasoning?
\end{itemize}

\subsection{Logical values}

\noindent We start with a tentative definition of the individual value of a neuron from the logical point of view, in our simple example:\\

\begin{defn}
	Let's give the value $1/2$ to every assertion of the form
	$1$ implies $E$ (here $1$ can be either $+1$ or $-1$ and independently $E$ can be replaced by $H$ or $EH$), the value $1/4$ for an assertion of the form
	$+1$ (resp. $-1$) implies $E\vee H$ (resp. $E\vee EH$, resp. $H\vee EH$). Then, by convention, the individual information
	value of the cell $y$ is the sum of the assertions it gives.
\end{defn}

\begin{rmk}
	\normalfont The individual value of a cell can be $0,1/4,1/2,3/4$.
	No cell can get the score $1$ because of the construction of the assertions by contraposition, from
	at most three implications of the form $a\Rightarrow \epsilon$, where $a$ is a condition and $\epsilon=\pm 1$.
\end{rmk}

\begin{rmk}
	\normalfont The above propositions indicate that the value of a group is not the sum of the values of
	the components. For the efficiency it even works in the wrong sense.
\end{rmk}

\begin{rmk}
	\normalfont In the above definition, we have decided that an atom, like $E$ is more precious than
	a union like $\neg E= H\vee EH$. This is apparently justified because we want to know the exact
	condition; however, proposition \ref{prop-triple} shows that it is more difficult to justify when thinking in terms of the neuronal assembly and of the
	collection of propositions regarding $E$, $H$ and $EH$. In fact, to find an algorithm which is able to decide the
	truth or not of every proposition about the three conditions is equivalent to an algorithm which can prove any of
	the atom when it is true, but it is also equivalent to an algorithm which can prove that they are false when it is the case.
	For instance $\neg E$ and $\neg H$ imply $EH$. That is because we are working in a Boolean logic, which gives us an
	important \emph{a priori} knowledge. In the Boolean setting there exists  a duality between prove and disprove.
\end{rmk}

\begin{rmk}
	\normalfont Note that the above duality doesn't totally disappear in intuitionist logic. Suppose we can prove $E\vee H$
	and $E\vee EH$, in a context where we know that $EH$ and $H$ are contradictory (i.e.  $H\wedge EH\Rightarrow\bot$), then we have a
	proof of $E$ or a proof of $H$, and separately a proof of $E$ or a proof of $EH$, then we have at least one proof of $E$.
	On the other side, if we know that $H$ is true, what can \emph{a priori} exclude that $EH$ is true if we don't have
	assumed they are contradictory.\\
\end{rmk}
\noindent In a general finite Heyting algebra, there exists a dissymmetry between the number of truth values of propositions and of truth values of
propositions of the form $\neg P$. Also, negative propositions are scarcer.
\begin{rmk}
	\normalfont The experiments we will report in the next sections,
	invited us, in particular with respect to a problem of classification, to measure the logical value of a cell or a group
	of cells, by the set of elementary propositions that a given activity exclude, i.e. conjecture to be false. The above definition 1
	accords with that, giving $1/4$ for each exclusion.\\
	This corresponds nicely with the notion of \emph{content}, that Carnap and Bar-Hillel studied in 1952. For a more complete
	discussion see the companion paper on probabilities and the forthcoming theoretical paper on semantic information.
	However, looking at more complex experiments, we will need a less rigid notion of information.
\end{rmk}

The preceding remarks justify that we orient ourselves on the search of a definition of logical value which is
more collective, i.e. concerns the whole layer, and which moreover, gives an equal value to the truth and falsity, i.e.
which concerns decidalibity.\\

\noindent In what follows, we consider a network which has learned, i.e. the weights are fixed, and collections
of vectors are given in the input layers, one for learning, one for testing, one for generalizing. Without contraindication,
the collection which is considered is the collection which was used for learning.\\

\noindent The intuition: the more the hidden layer can easily deduce the condition $a$ in the set $A$ from its activity, the higher
its logical information quantity is.\\
Note this is the point of view of an observer on the layer, knowing the receptive field of each neuron, i.e. the manner
this neuron reacts to every input during the learning, or testing or generalizing. This implies no obligation
for the network by itself, which is working as it prefers. Thus this measure of information has to be completed by
the analysis of the weights, to pass from this inner layer to the following ones, and the manner these weights take
care of the logical content, in order to achieve the role assigned to the networks, here in the final layers.

\begin{rmk}
	\normalfont In the next section, with a little more complex experiment, we will compare statistically the weights
	and the logic, and we will see that they perfectly agree, showing that the network elaborate proofs through the weights.
\end{rmk}

\indent In more general contexts, the above three conditions are replaced by some discrete variable of interest,
in a given set, as in ordinary classification. We describe them by the truth or no of some statement in a formal
language $\mathbb{L}$, i.e. under some declaration of types $X, Y,...$ and variables $z_X, z_Y, ...$, a proposition $P$, and so on. The objective is to
decide if yes or no a subset of propositions is true or false.\\
In the simple example below, the propositions described the Boolean calculus over $A_3$, with three elements.\\
\noindent The set of propositions of interest is supposed finite and closed by opposite, it is written $\mathcal{P}$ and
we want to know if they are true or not, i.e. prove or disprove.
\begin{rmk*}
	\normalfont This can be embedded in an intuitionist framework, because we are not forced to ask that necessarily, $P\vee\neg P$ is true.
\end{rmk*}
\
\noindent In the simplest example we considered below, $\mathcal{P}$ is the whole algebra of subsets of $A_3$, except the empty set.
\begin{rmk*}
	\normalfont If we include the conditioning by angular intervals, as does the model $z$ with $L_1$, $L_2$ and even a part of $L_3$, the
	information is always maximal. Then something finer has to be taken in account, which is the economy of the theory. Here, the restriction to the
	Boolean algebra plays this selective role.
\end{rmk*}

\noindent First we saw the important role played by a \emph{quantization} of the receptive fields of the neurons. This is a non-trivial
point, because discrete or continuous is also a matter of observation, or level of description.
Presently we don't enter into this difficulty, and we assume that for all the cells in the layer, it is possible to decide
if they have a good quantization or not, in function of the concerned properties to be proved or disproved.\\
And for simplicity we assume the quantization is binary, $-1$ or $1$, as in the example. It is not a big difficulty to extend
the discussion to several disjoint intervals in $[-1,+1]$.\\
The decision of the cell is given by the number $+1$ or $-1$. For the cells which are uncertain, we could add a $0$, but we will see in one minute how to give them an information zero.\\
Letter $C$ denotes a set with elements $+1$ and $-1$ and perhaps $0$. *It will be integrated in the formal language $\mathbb{L}$ by adding a type $C$.*\\
The quantized activities of the layer $L_k$ are represented by a subset of the product $C^{n}$ when the layer contains $n$ neurons.\\

\indent Then we attribute to the layer $L_k$ a vector of propositions $I_k$ which is made as follows:\\
for the neuron $y$, two propositions of the form $1\Rightarrow P_y$ and $-1\Rightarrow Q_y$, in the
language $\mathbb{L}$, which could be noted $P^{+}_y$ and $P^{-}_y$ respectively.\\
If the neuron is uncertain (i.e. telling uh or $0$) we adopt the convention that its $P_y=Q_y=\top$.\\
\indent Each possible quantized activity $\epsilon$ is described by a vector of coordinates $\epsilon_y$
in the set $C={-1,0,1}$. Then it defines a family of propositions $P^{\epsilon}_y, y \in L_k$. These propositions constitute $I_k^{\epsilon}$;
they
can be understood as the axioms of a theory $\mathbb{T}^{\varepsilon}$.\\

At least three collections of theories are interesting to consider: 
\begin{enumerate}[label=\arabic*)]
	\setcounter{enumi}{-1}
	\item the full family $\mathcal{T}^{0}_k$, corresponding to all the vectors $\epsilon$,	without exception,
	\item the sub-family $\mathcal{T}^{1}_k$ made by the \emph{consistent} theories only (i.e. without contradiction),
	\item the sub-family $\mathcal{T}^{2}_k$ of the preceding corresponding to the vector that can really happen in the layer, given the set of vectors in the input layers (learning, testing or generalizing). 
\end{enumerate}
The second one is more convenient than the first, because inconsistence is not comfortable, but is it preferable to the third one? The sets of inputs is difficult to describe, however its properties are determinant in all the
applications, and the role of the network is to define (or extract) a structure from the data, which allows it to generalize the efficient functioning
to other data. Moreover, we saw before that efficiency depends on a set of data. Then it is probably much better to work with the third species of sets of axioms.\footnote{if "everything works as it should be", the collection of these families $\mathcal{T}_k$ for all the layers forms an object in the topos of the network.}
\begin{defn}
	The minimal logical information of the working layer $L_k$ is the minimum of the ratio between the number of propositions that can be
	logically deduced from the axioms of any theory $\mathbb{T}^{\varepsilon}_k$ belonging to $\mathcal{T}^{2}_k$, and the cardinal of the consequences of the wanted propositions.
\end{defn}

\noindent In our simplest example, we can compare the input layer with the inner layer, and also compare the different models, $z$ versus $z'$, then
two hidden layers versus one, then three layers versus two. From the above discussions, they are evidently disposed in a growing order of logical information.
(The only ambiguous case is $z$ with a too fat $L_3$, which is difficult to compare with the model $z'$, because they don't have the same defects.)\\

\indent As the above propositions $2,3,4,5$ show, it would be nice to have a more localized notion of information, involving the set of subsets
of $L_k$. (We will introduce such a definition below, involving together logic and probability.)\\
\indent In reality, the practical information value is not only given by the whole collection
of theories, for instance the set $\mathcal{T}^{2}_k$, it must also take in account  the collection of demonstrations of the propositions
of interest or their opposite, starting from the concrete axioms $P^{\varepsilon}_y, y \in L_k$, and measure their difficulty, number of branches
of trees in a proof, and number of useful initial propositions, then here, the number of cells which are involved comes into the play.\\
\indent A kind of Galois theory could exist in this context, describing
how these sets and proofs are changed by adding (or deleting) a certain set of logical cells.\\
\indent This has to do with the stability of the deduction, which has also its importance, practical and theoretical. Two aspects of stability appear naturally:
\begin{enumerate}[label=\arabic*)]
	\item the deletion or dysfunction of few cells can destroy the information value or not;
	\item the function can be easily recovered or not by a few re-learning or training.
\end{enumerate}

These two aspects being probably not independent.

\section{Simple networks for doing predicate calculus}\label{sec-predicate}

In this second experiment, we test the hypothesis that a simple network of few layers is also able to decide between propositions involving
existential and universal quantifiers, $\exists$, $\forall$, and when doing that, spontaneously constructs cells that perform logical analysis in the
inner layers, having the same kinds of properties than the logical cells of section \ref{sec-first-exp}: they introduce discrete responses, in such a
manner that spiking or not spiking implies propositions, from which it is possible to answer the final question by logical deductions.\\

\noindent The goal of the DNN in this case is to recognize if two different objects are disjoint, intersecting or in relation of inclusion.
It can be images with two colored rigid objects in a space, or two different voices pronouncing sentences.\\
In the spatial case, the input data are collections of dots in two possible colors, in the temporal one they are collections of sounds
in two possible timbers.\\
A first layer $HL_0$ is made of neurons detecting only one color, or only one timber, a kind of transducer.\\

\noindent Our experiment shows that a fairly successful network exists with only one hidden layer, but for observing logical cells, and for the aptitude of generalization,
the network must have at least two or three layers, depending on the nature of the space or time interval.\\

\subsection{A first experiment in two forms}

\subsubsection{The experimental setting}

\noindent For simplicity we speak of images and colors, named Red and Green, and we consider only homogeneous one dimensional spaces $D$: one is a circle, one is a segment;
both have a discretization of the order of $100$ dots or unit segments (See Figures \ref{subf-interval-linear} and \ref{subf-interval-circular}).\\

\noindent
\begin{figure}[ht]
	\begin{subfigure}{.62\textwidth}
		\centering
		\includegraphics[width=.9\linewidth]{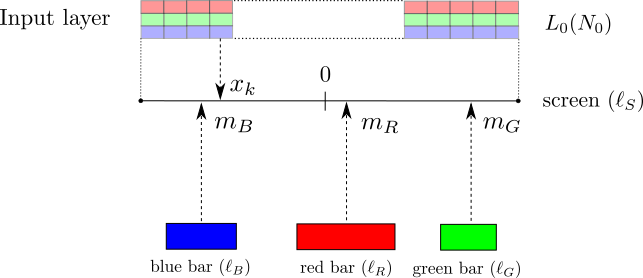}
		\caption{Experiment on a line}
		\label{subf-interval-linear}
	\end{subfigure}
	\hfill
	\begin{subfigure}{.36\textwidth}
		\centering
		\includegraphics[width=.9\linewidth]{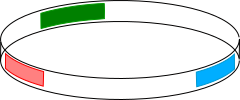}
		\caption{Experiment on a circle}
		\label{subf-interval-circular}
	\end{subfigure}
	\caption{Input Layer for $3-$bar experiments}
\end{figure}

Rigidity of the objects $R$ and $G$ means that, for each image and each color the dots form a segment of constant length. These lengths are of the order of $3$ or $5$ out of $18$ in the linear case, as in the circular case.\\

\noindent Important: in this first experiment, we do not change the length of the objects, then one of them, say the red,
can never be included into the other, say the green, but the green object can be strictly included into the red one.\\

\noindent
\begin{figure}[ht]
	\centering
	\includegraphics[width=.55\linewidth]{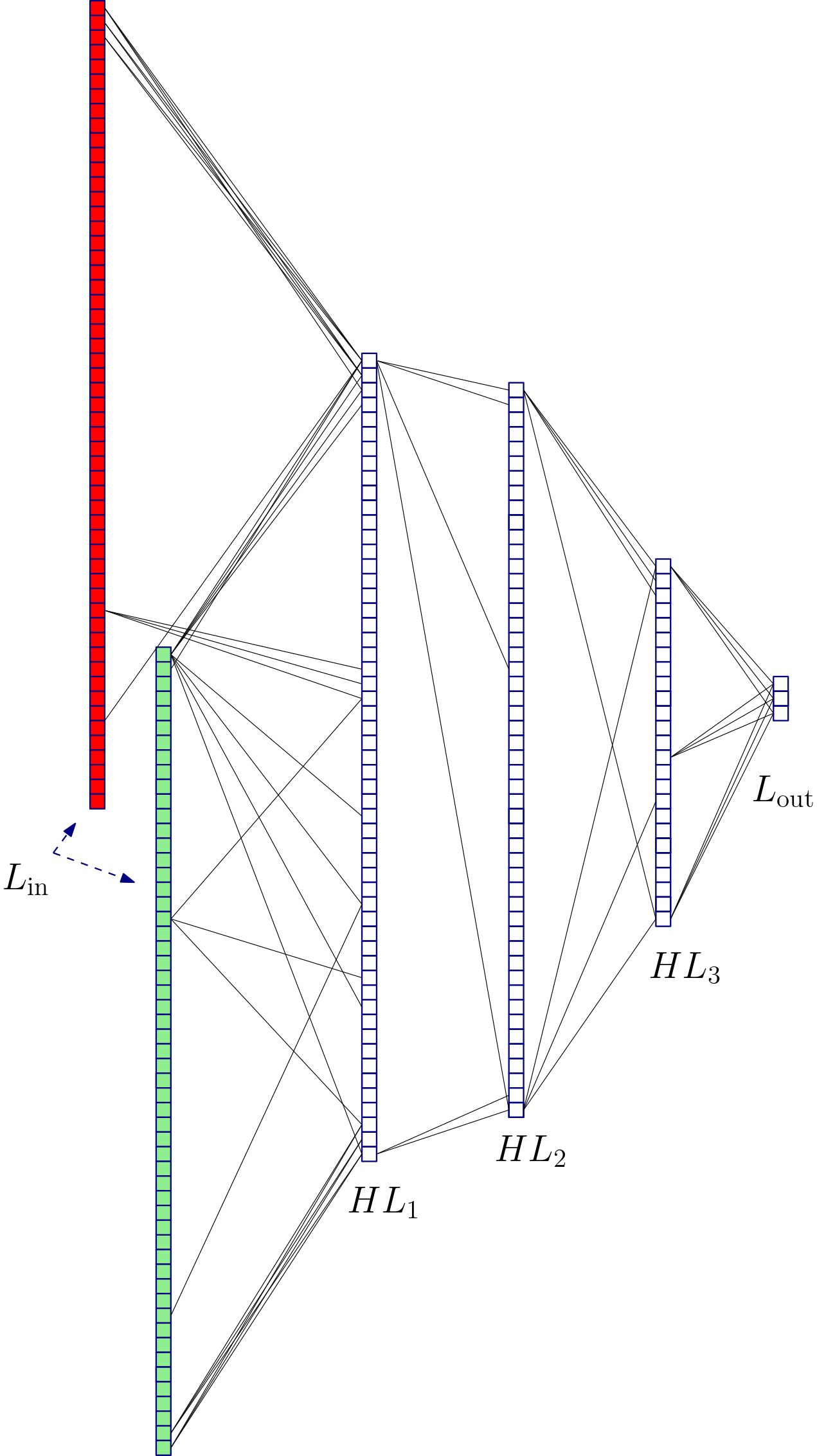}
	\caption{Architecture of the network for the predicate experiments}
	\label{fig-archi-exp2}
\end{figure}

\noindent As it is illustrated in Figure \ref{fig-archi-exp2}, a first layer $HL_0$ contains two families of neurons, each one having $55$ elements (a number coming from the first experiment) for detecting one of the two colors. By convention, we say they detect red or green dots. This layer is not considered as a hidden layer, it performs a transduction, as the cones do in the retina.\\
We considered two types of receptors, the simple one is Gaussian or wrapped Gaussian (theta distribution), the complex one is made by a difference of two Gaussian curves or wrapped Gaussian curves, introducing a negative answer when the color disappears of the receptive field, then detecting the contrast. The type of receptor influences the learning and the generalization, but it appeared that the simple one has better performance.\\

\begin{rmk*}
	 \normalfont The addition of a layer which locally combines the activities of $HL_0$ has a negative effect as well.
\end{rmk*}

\noindent The first truly hidden layer $HL_1$ has $55$ cells (to have the possibility of comparing with section \ref{sec-first-exp}).\\
When we add a second hidden layer $HL_2$,
it will have $50$ neurons, and when we add a third one $HL_3$ it will have $25$ neurons, except with some mentioned exception). See Figure \ref{fig-archi-exp2} for an illustration. \\

\noindent The last layer has three neurons, named $D$ for disjunction, $IO$ for intersection without inclusion (intersection only) and $II$ for inclusion.\\

\noindent The mapping (to be learned) from a layer to the next one, is of the type $\tanh(C\sum x)$, with $C$ equals $1$ or $2$. In all cases $C=2$ gives better results.\\
Between the last hidden layer and the output layer, we choose a
linear mapping, without using the $\tanh$ activation function, and to normalize the result as a probability, adapted to the cross entropy. The main reason is that this gives better results.

\noindent A metric was chosen on the last layer, for measuring the accuracy of the answer, then for the training phase. The choice of this metric has a strong influence on the results.\\
\indent We already saw this point in the first experiment, but in this case, to ontain a good performance, either in accuracy either for the logical behavior, it is not sufficient to respect the symmetry between the three points, for instance by using a two dimensional coding. In fact all the results below need the use of the cross-entropy, i.e.  the Kullback-Leibler distance between the feedback normalized beliefs (given by the network) in the three options and the right one (non random). 
Thus the quantity to minimize is $-\log p_i$ for the condition $i$; $i=1,2,3$ for $D,IO,II$ respectively.\\
This metric is known to improve most classification problems \cite{Ghosh2017RobustLF}; it is particularly adapted to connecting semantic information and statistical information. 

In both the linear and circular cases, a dramatic improvement from a quantitative point of view appeared with two hidden layers instead of one. IT is the case for both the minimal loss function after training and the number of residual errors. Note that this last number becomes stable with two hidden layers (around $1/100$ for the circular case and for the linear case). We will discuss this limit later, however it obviously represents the limited precision of the receptive fields, which makes them unable in many case to distinguish the intersection or the inclusion versus intersection only, when the boundaries of the objects are close. But it was
not the goal of this study to improve the performance in this direction.

\subsubsection{Theory, predicative cells, conclusive or not}\label{conclusive}

In order to describe logical cells, we constructed for each neuron $i$, a receptive field to an intermediate proposition $P$.
\begin{enumerate}
	\item For local propositions : for every point $a\in D$, the distribution of the responses of the cell $i$, when Green and Red appear together at the position $a$, noted $R(a)\wedge G(a)$, the same for $R$ but not $G$ appearing in $a$, denoted $R(a) \wedge \neg G(a)$, the same for $G$ but not $R$ appearing in $a$, noted $G(a) \wedge \neg R(a)$, and finally the same for $\neg R(a) \wedge \neg G(a)$.
	\item We also consider global propositions, describing the reaction of the cell $i$ when the presented objects are disjoint (condition $D$), when they intersect without inclusion (condition $IO$, intersection only) and when the green object is included in the red one (condition $II$).\\
	\noindent Proposition $D$ is $\forall x\in D, \neg (R(x)\wedge G(x))$.\\
	Proposition $II$ is $\forall x\in D, G(x)\Rightarrow R(x)$.\\
	And proposition $IO$ is $(\exists x\in D, G(x)\wedge R(x))\wedge (\exists y\in D, G(y)\wedge\neg R(y))$.
\end{enumerate}

All these distributions were represented by a color code (rectangles for
the local questions, segments for the global ones), blue for $-1$, red for $+1$,
and barycenters of the colors for activity in between $-1$ and $+1$ (see Figure \ref{fig-cell25}). 
\noindent
\begin{figure}[ht]
	\begin{subfigure}{.62\textwidth}
		\centering
		\includegraphics[width=.95\linewidth]{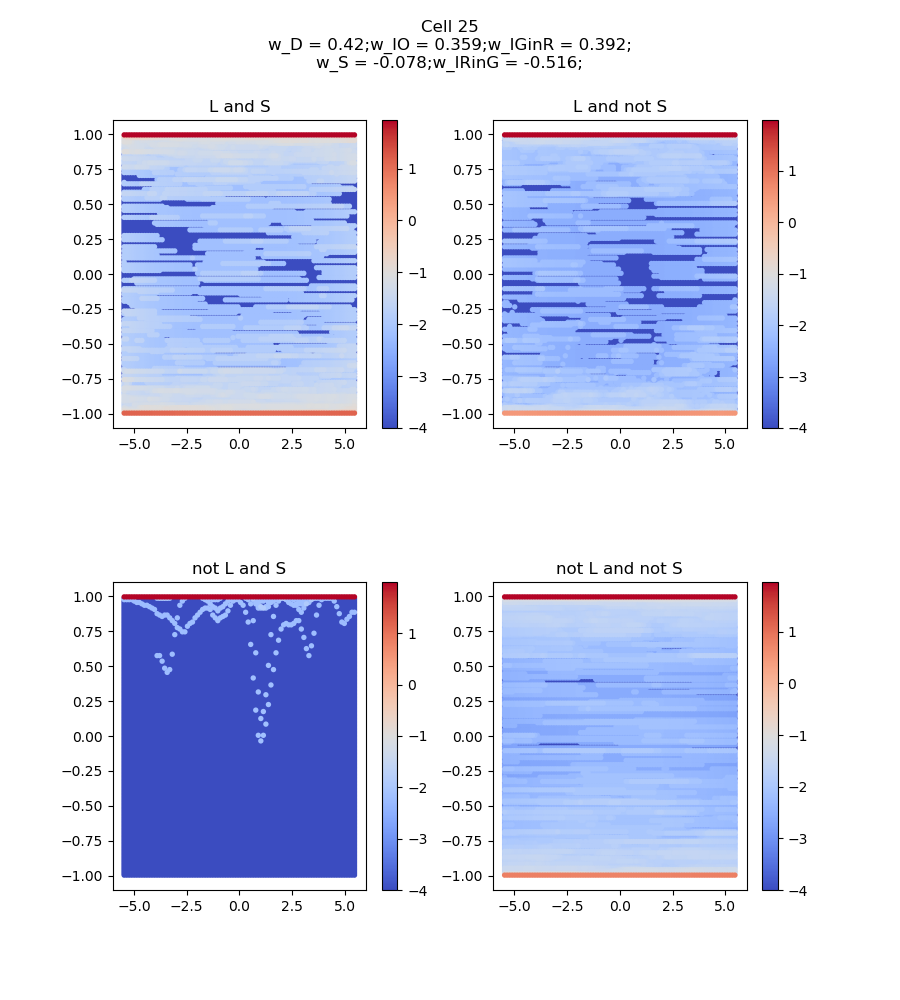}
		\caption{$HL_3$ Cell 25 - Activity Distribution}
		\label{subf-cell25-act-dis}
	\end{subfigure}
	\hfill
	\begin{subfigure}{.36\textwidth}
		\centering
		\includegraphics[width=.95\linewidth]{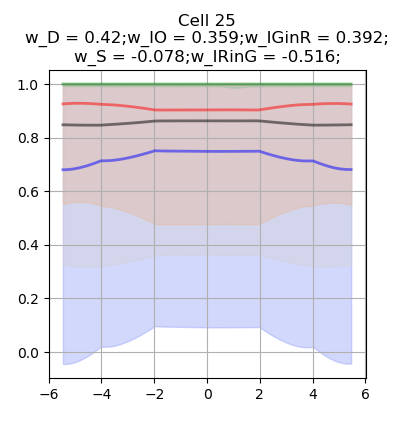}
		\caption{$HL_3$ Cell 25 - View}
		\label{subf-cell25-view}
	\end{subfigure}
	\caption{Cell $25$ in layer $HL_3$}
		\label{fig-cell25}
\end{figure}

In addition, we took advantage from the fact that the input image can be fully described by two bounded real parameters, which are the positions of the
two centers of the intervals. In the circular case they are two angles, giving a point in a flat torus; in the linear case, this gives a point in a square.
Therefore the complete activity of the cell as a function of the input, can be represented by a colored square, with the above color code. We called this
representation the \emph{raw activity} of the individual cell (see Figure \ref{fig-cell2-rawact} for the linear case). 

\noindent
\begin{figure}[ht]
	\centering
	\includegraphics[width=.55\linewidth]{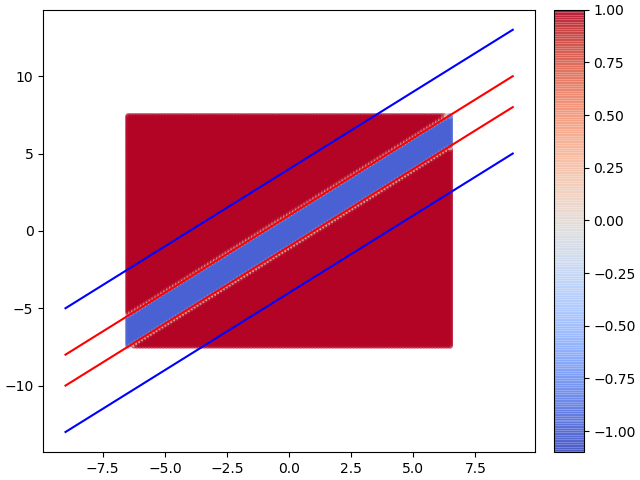}
	\caption{Linear case - $HL_3$ Cell 2 - Raw Activity}
	\label{fig-cell2-rawact}
\end{figure}

It appears that the most readable representation is by far the raw activity, but the other representations give a finer idea of the
variability and allow to confirm what appears on the raw activity.
\begin{rmk*}
	\normalfont All results are very noisy with only one hidden layer $HL_1$, but become  very intelligible with two hidden layers.
\end{rmk*}
	In the circular case, we can see a Fourier analysis on the torus in $HL_1$, as it is illustrated in Figures \ref{subf-Fourier15} and \ref{subf-Fourier37}, which is pursued in part in $HL_2$, where it also
	appears almost discretized cells for the three propositions $D$, $II$ and $IO$. In $HL_1$ the Fourier analysis is made separately on the two middle angles $\theta_R, \theta_G$, but in $HL_2$ the analysis is done with respect to the phase difference $\theta_R- \theta_G$, which is an evident progression with respect to the	logic, allowing the individual cells to represent the characteristic functions of the objectives.\\
	For the linear segment case, we got the same behavior, with a kind of partial Fourier analysis corresponding to the action of $\mathbb{Z}/2\mathbb{Z}$ by symmetry around the middle of the full segment, but we observed an important difference with respect to the circular network, because, with two hidden layers, no cell corresponds to $IO$.

\noindent
\begin{figure}[ht]
	\begin{subfigure}{.48\textwidth}
		\centering
		\includegraphics[width=.95\linewidth]{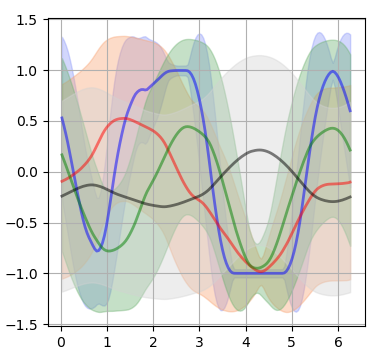}
		\caption{Fourier analysis of cell 15}
		\label{subf-Fourier15}
	\end{subfigure}
	\hfill
	\begin{subfigure}{.48\textwidth}
		\centering
		\includegraphics[width=.95\linewidth]{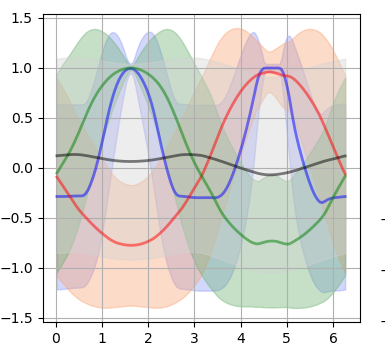}
		\caption{Fourier analysis of cell 37}
		\label{subf-Fourier37}
	\end{subfigure}
	\caption{Fourier analysis of $HL_1$ in the circular case}
\end{figure}

We concluded that the $IO$ cells in $HL_2$ for the circular case with two hidden layers is almost certainly a consequence of the good approximation of the characteristic function of this proposition by an harmonic of degree $3$.\\

What is amazing, is that the Fourier analysis completely vanishes when we introduce a third hidden layer $HL_3$, a consequence is the total disappearance of cells interested in $IO$, intersection only. This is true in both the circular and linear case.\\
We conclude that, with three hidden layers, something in the training appears, which allows neurons to get access to higher frequencies, for representing step functions, that where too difficult to represent with the first harmonics
of the Fourier analysis.\\
With three hidden layers, the cells in $HL_2$ are still affected by noise but fully intelligible for the two conditions $D$ and $II$. In $HL_3$ the representation is perfect, fully quantized at $-1$ and $+1$, without any cell $IO$. We say that a cell is of type $D$, resp. $II$, resp. $IO$, if conditioned by an input of the respective type it tells $\pm 1$ and conditioned by an input of the two other types, it tells $\mp 1$. \\
A consequence is a nice predicative logic: with one $D$ cell and one $II$ cell, the conclusion of the output is accessible. Suppose for simplicity that
both cells prefer $+1$ for their respective type, then the pairs of possible activities $(+1,-1)$, $(-1,-1)$, $(-1,-1)$ correspond respectively to
a prediction $D$, $II$, $IO$.\\ 

The fact that the cells quantize at two opposite values has the consequences that the receptive fields for the propositions
$G(a)\wedge R(a)$ and $\neg R(a)\wedge G(a)$ saturate at $-1$ or $+1$, the first one corresponding to the condition $D$, the second one to the condition $II$.\\

\textbf{In numbers:}
\begin{enumerate}
	\item in the circular case, with two layers, in $HL_1$ it is difficult to detect a logical functioning, but in $HL_2$, with $50$ neurons, the Fourier analysis in $\theta_R- \theta_G$
	reconstructs fairly good raw activities, giving $20$ cells of type $D$, $15$ cells of type $II$ and $13$ cells of type $IO$, and two strange cells. Cf. Figures ... \\
	We tried also with a second layer of $25$ neurons and got the same kind of spectrum: $6$ cells $D$, $8$ cells $II$ and $11$ cells $IO$.
	\item With three hidden layers, $HL_2$ seems at first sight to resemble the preceding, but it contains no cell of type $IO$, $26$ cells $D$ and $24$ cells $II$.
	The layer $HL_3$ develops an impressive quantization, giving $15$ cells $D$ and $10$ cells $II$.
	\item In the linear case, things are less imaginative, certainly because a continuous Fourier analysis is missing.\\
	With two layers, in $HL_1$ we recognize $18$ cells corresponding to an asymmetric representation of $\mathbb{Z}/2\mathbb{Z}$ (Figure \ref{fig-Fourier-asymmetric}), and $32$ cells to a symmetric one (Figure \ref{fig-Fourier-symmetric}). In hidden layer $HL_2$, $21$  $D$ cells, $22$ cells $II$, $4$ cells Fourier symmetric, $1$ asymmetric, and $1$ not interpretable for us. 
	\item  With three hidden layers, in $HL_2$ we found $17$ $D$ cells, $22$ $II$ cells,  $5$ symmetric Fourier cells, $6$ asymmetric ones, and in $HL_3$, no Fourier cell, $15$ cells of type $D$ (Figure \ref{subf-HL3-type-II}) and $11$ of type $II$ (Figure \ref{subf-HL3-type-II}).
\end{enumerate}

\noindent
\begin{figure}[ht]
	\begin{subfigure}{.48\textwidth}
		\centering
		\includegraphics[width=.95\linewidth]{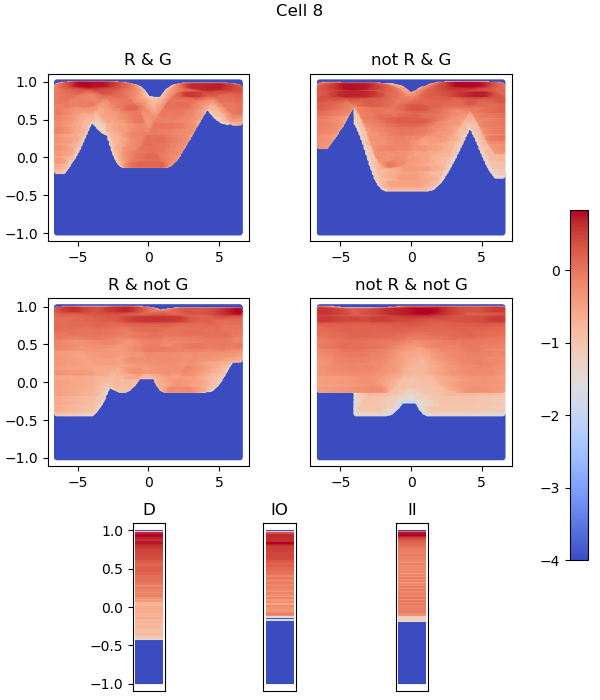}
	\end{subfigure}
	\hfill
	\begin{subfigure}{.48\textwidth}
		\centering
		\includegraphics[width=.95\linewidth]{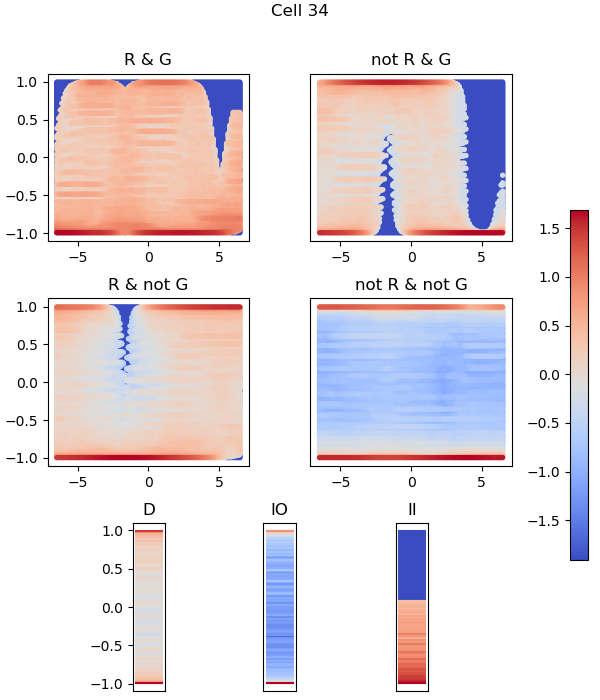}
	\end{subfigure}
	\caption{Fourier analysis on $\mathbb{Z}/2\mathbb{Z}$: the asymmetric case}
	\label{fig-Fourier-asymmetric}
\end{figure}

\noindent
\begin{figure}[ht]
	\begin{subfigure}{.48\textwidth}
		\centering
		\includegraphics[width=.95\linewidth]{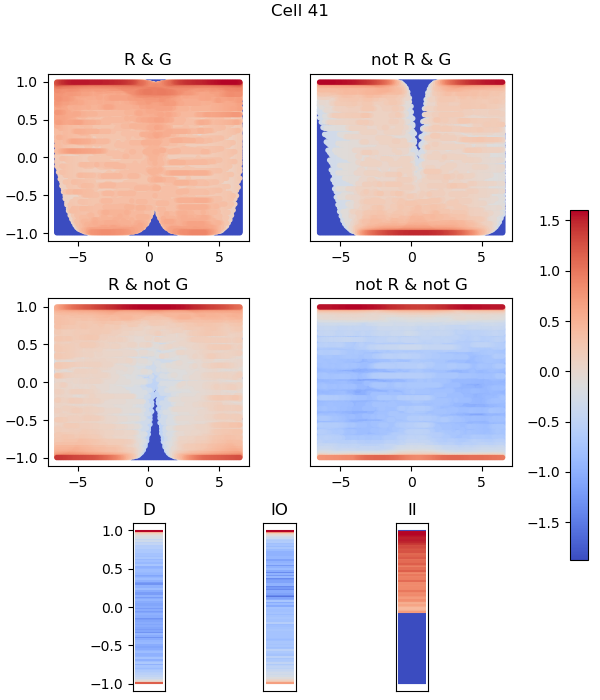}
	\end{subfigure}
	\hfill
	\begin{subfigure}{.48\textwidth}
		\centering
		\includegraphics[width=.95\linewidth]{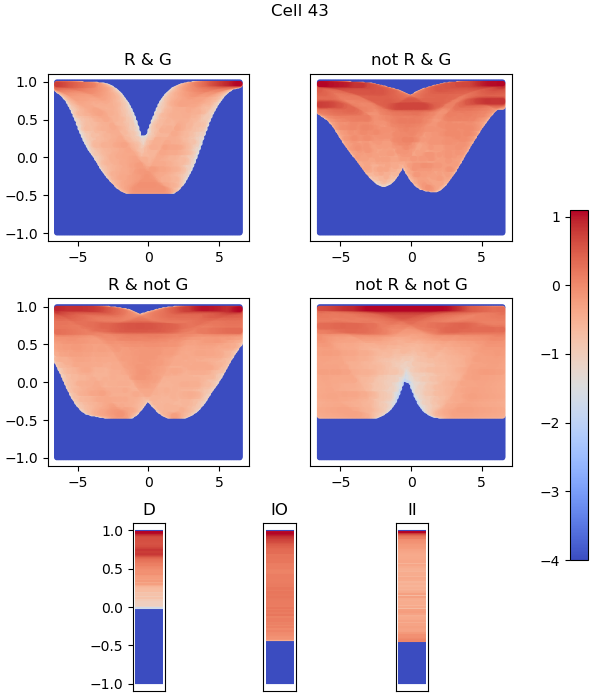}
	\end{subfigure}
	\caption{Fourier analysis on $\mathbb{Z}/2\mathbb{Z}$: the symmetric case}
	\label{fig-Fourier-symmetric}
\end{figure}

\noindent
\begin{figure}[ht]
	\begin{subfigure}{.48\textwidth}
		\centering
		\includegraphics[width=.95\linewidth]{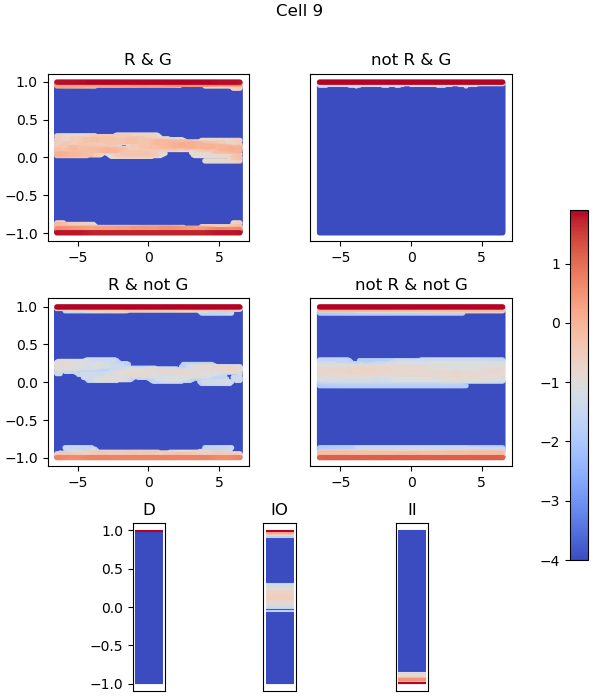}
		\caption{$HL_3$ - Cell of type $II$}
		\label{subf-HL3-type-II}
	\end{subfigure}
	\hfill
	\begin{subfigure}{.48\textwidth}
		\centering
		\includegraphics[width=.95\linewidth]{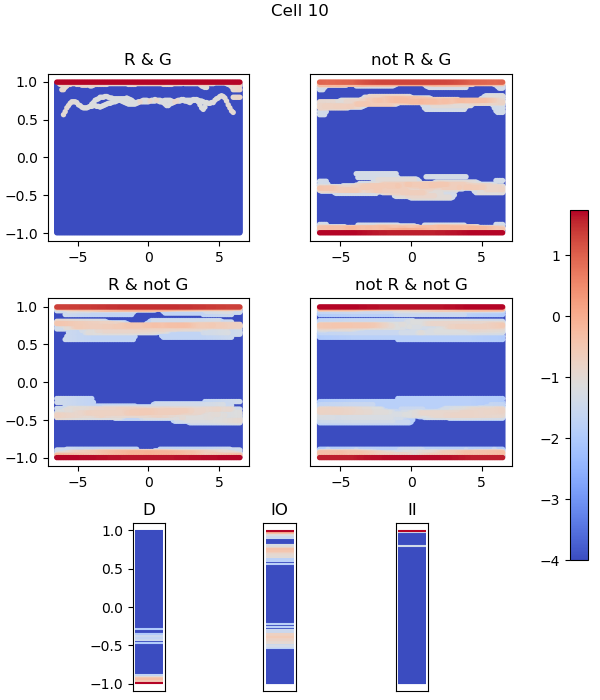}
		\caption{$HL_3$ - Cell of type $D$}
		\label{subf-HL3-type-D}
	\end{subfigure}
	\caption{$3$ hidden layers - $HL_3$}
\end{figure}

\textbf{Important remark:} the distribution of preference does not reflect at all the statistics of the imputs, which are apparent on the raw graphics. In fact, the number of $D-$cases is larger than the number of $IO-$cases, which is much lager than the number of $II-$cases.	There is a tendency to the equilibrium between $II$ and $D$. We plan to study the variability of these populations in a further study. 

\noindent We saw, in this experiment, which was expected to test predicate logic, the same kind of results we saw in the first experiment, with ordinary propositional calculus. Clearly, the Boolean logic at the output dominates. And we could conjecture  that this will happen in any classification problem. However, this is not the full picture. This is clear for the condition $IO$, because its "mature" treatment relies on an indirect logical reasoning, the proposition never being directly accessible. This reminds the hidden predicate calculus, where the expression of this proposition is twice more complex than the other ones, $D$ and $II$.\\
Of course, we must be conscious that the above networks learn by minimization of a certain functional $F$. They learn to approximate the desired responses,
and all that is fully supervised. The logic, at least, results from the analytical properties of $F$ and the nature of the images.
Therefore we can interpret the difficulty of accessing $IO$ by the difficulty to represent a characteristic function which is more complex than the other ones; this complexity is precisely the complexity of the logical formula in predicate calculus.
However, we will see below that a bit more solicited network can success in representing sometimes this function for $IO$, and we never see a mixture between it and the characteristic functions of $II$ or $D$. It agrees with the quantization, and perhaps depends on it. The combination is made by logics, not by interpolation. Thus we can at least suggest that the minimum of the
functional $F$ with $3$ layers and more, has a logical flavor, and that this logic is influenced both by the output classification and by the predicative formulas with respect to the localized input, i.e. the composite nature of the objects in the scene, and the composite nature of their reciprocal relations.

\subsubsection{Tests of generalization}

The ability to generalize was tested on three different sets of data: changing the lengths from $3$ and $5$ to $4$ and $6$ out of $18$ with and without an invertion of the colors, and a last one, just exchanging the colors without changing the lengths.\\
\indent The results were not so bad: $12\%$ of errors for the change of lengths
respecting the colors, and the same for the exchange of colors without changing the lengths, but surprisingly only $8\%$ of error for both the change of length
and the exchange of colors. This is as if the network were more perturbed  when correlations that it had established by itself are violated in the new data.\\

\noindent \textbf{Important:} in all cases the raw activities are similar to the raw activities without generalization, only slightly deformed. However, very nicely, when the small and large bars changed their color in the generalization test with respect to the learning condition, the cells that were attached to saturation of $R \wedge\neg G$ (resp. $G\wedge \neg R$) became almost saturated for the other, $G\wedge \neg R$ (resp. $R \wedge\neg G$). In fact these cells had no notion of color, they were interested by the comparison of lengths, and the local property $S(a)\wedge\neg L(a)$, $S$ for small and $L$ for large.

\noindent Therefore, the network gave the impression that it had understood by itself that what is important, to decide about the inclusion or not, was to distinguish the short object from the large one!\\

This was a reason to modify the problem, with data of several lengths and inverting the colors.\\

\subsection{A second experiment with a richer learning}

In the above experiment, the length of the objects did not vary, and we saw a not too bad but limited capacity of generalization. We decide, then, to consider a larger collection of images (or conversations between two persons) where the objects in red or green (or the sentences)
can change their lengths.  The conjecture is that the network will also succeed, constructing by itself in $HL_3$ predicative cells of different types,
able to conclude by logical proofs.\\
In this set of experiments as well, we consider colored objects in a one dimensional space. And we present the results for the linear segment, not the circle.\\

\indent The experimental setting is the same as before (see Figure \ref{fig-archi-exp2}), a layer $HL_0$ with two populations of small Gaussian sensors, one for green and the second one for red;
then a layer $HL_1$ with $55$ cells, a layer $HL_2$ with $50$ cells, a layer $HL_3$ with $25$ cells. The activation function is $\tanh (C\sum X)$, with $C=1$ or $C=2$, with the mentioned exception of the mapping from $HL_3$ to the output $HL_4$, which is linear and normalized as a probality. The overall metric is the cross-entropy. In the back-propagation algorithms, we vary the sizes of batches and the number of iterations.\\
However, the input images now contain bars of several lengths, for instance $2,4,6,8$ or $3,5,7$ of both colors out of $23$ units for the diameter of the space. Discretization is still $100$ for the total space.\\

\begin{rmk*}
	\normalfont A priori, coincidence of the objects in space can happen. However, we observed that the network was unable to detect this case, when we included it in the objectives. We supposed that this was due to the few data where coincidence occurs. We then repeat the coincidences in such a manner that it happens as often as other situations. The result is good for the accuracy, but mostly destroys the logical behavior. Therefore we decided to avoid coincidences. We could have decided to include them without asking about them, but in this case, the questions about the inclusion of red in green or the converse would have conflicted.
\end{rmk*}

We asked four final questions: $D$ disjunction, $IO\equiv (G\cap R\neq\emptyset) \wedge (G\cap R^{c}\neq\emptyset)$, $II_R\equiv R \subsetneq G$, $II_G\equiv G\subsetneq R$.\\
We did that with four lengths, $2,4,6,8$ for the two objects, and tested for $3,5,7$ for both (out of a total length $23$), with the same discretization as before, $100$.\\
The network in this case gave the best results we got from the beginning, with a testing as good as the training, around $1\%$ accuracy. Moreover, the predicative cells were excellent, showing a low variability in $HL_2$ and being very well quantized in $HL_3$, corresponding to some partitions of the following set of ten propositions $D\wedge (L=R)$, $II\wedge (L=R)$, $(R \subset G)\wedge (L=R)$, $(G\subset R)\wedge (L=R)$, $D\wedge (L=G)$, $II\wedge (L=G)$, $(R \subset G)\wedge (L=G)$, $(G\subset R)\wedge (L=G)$, and surprisingly $IO\wedge (L=R)$ and $IO \wedge (L=G)$.\\
This means that the network has \textsl{understood} the existence of red and green objects, not only the fact that one is small and the other one is large.\\

However, and it is a fundamental result: all the propositions that were decided (proved or disproved) by the individual cells in this experiment (as in the other one) belong to the algebra generated by the four propositions forming the objectives: $D$, $IO$, $R \subset G$, $G\subset R$.\\

\noindent \textbf{In numbers:} In $HL_2$, containing $50$ cells, we observed $20$ cells of type $D$, $6$ of type $II=(G\subset R)\vee (R\subset G)$, $7$ of type $R\subset G$, $13$ of type $G\subset R$, $1$ of type $IO$, (but vaguely) and three cells difficult to interpret, being saturated at one value or perhaps doing Fourier analysis on $\mathbb{Z}/2\mathbb{Z}$.\\ 	
In $HL_3$, we observed $9$ $D$ cells, $5$ $II$ cells, $0$ $IO$ cells, $3$ cells $R\subset G$, $6$ cells $G\subset R$, no bizarre cells, but a wonderful pair of cells, number $1$ and $25$ (See Figure \ref{fig-Exceptional-pair-Raw}):\\
\noindent cell $1$ tells $+1$ when $L=R$ and $\neg D$ are true or when $L=G$ and $IO$ (here very net) are true, it tells $-1$ when $L=R$ and $D$  are true or when $L=G$ and $\neg IO$ are true, then when it spikes at $+1$ we know $(IO \wedge (L=G))\vee (L=R)\wedge ((IO)\vee (G\subset R))$ which is equivalent to $IO \vee (G\subset R)$, and when it spikes at $-1$, we know $(D\wedge (L=R))\vee (L=G)\wedge (D\vee(R\subset G))$, which is equivalent to $D\vee (R\subset G)$; the cell $25$ does the analog, but exchanging $+1$ with $-1$ (which has no importance) and $R$ and $G$, which has an importance, because the conclusion of its "reasoning" is as follows: $+1$ implies $D\vee (G\subset R)$, $-1$ implies $IO\vee (R\subset G)$.\\ Thus, considering the pair of neurons $1$, $25$, we have $(+1,+1)$ implies $G\subset R$, $(+1,-1)$ implies $IO$, $(-1,+1)$ implies $D$ and $(-1,-1)$ implies $R\subset G$. \textbf{This pair completely solves the classification problem}.\\

\noindent
\begin{figure}[ht]
	\begin{subfigure}{.48\textwidth}
		\centering
		\includegraphics[width=.95\linewidth]{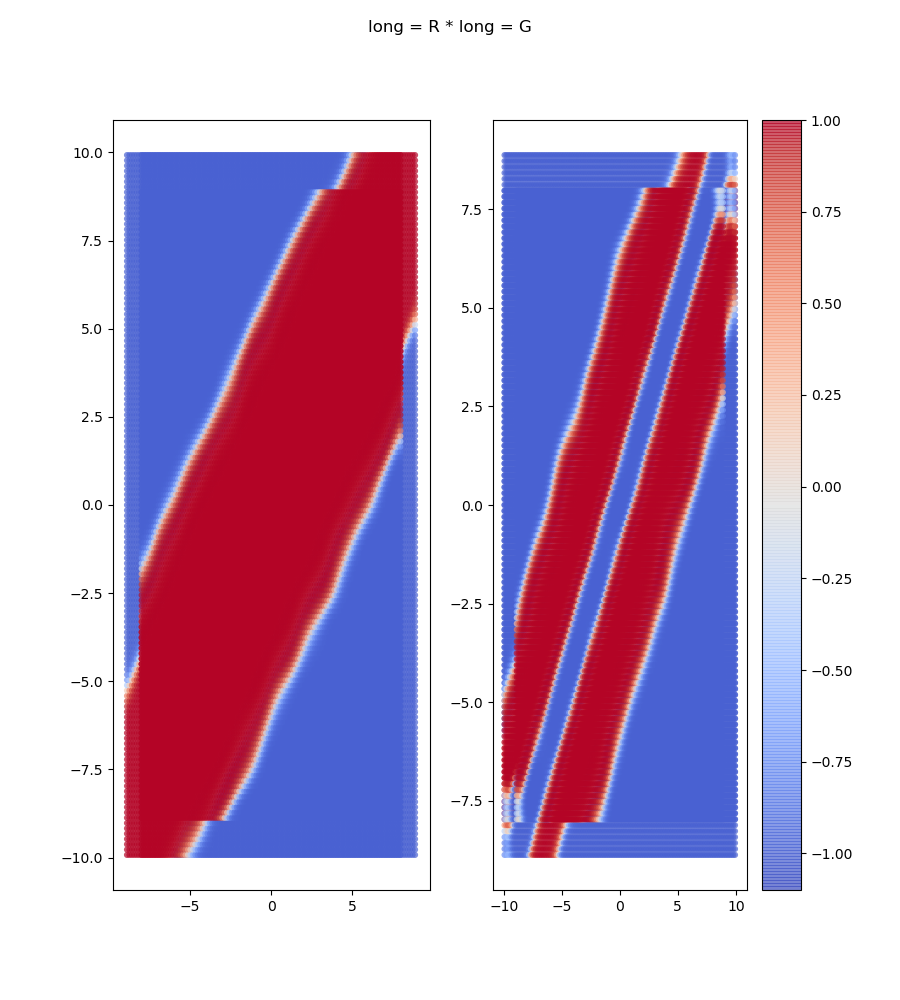}
		\caption{$HL_3$ - Cell $1$}
		\label{subf-HL3-Rawcell1}
	\end{subfigure}
	\hfill
	\begin{subfigure}{.48\textwidth}
		\centering
		\includegraphics[width=.95\linewidth]{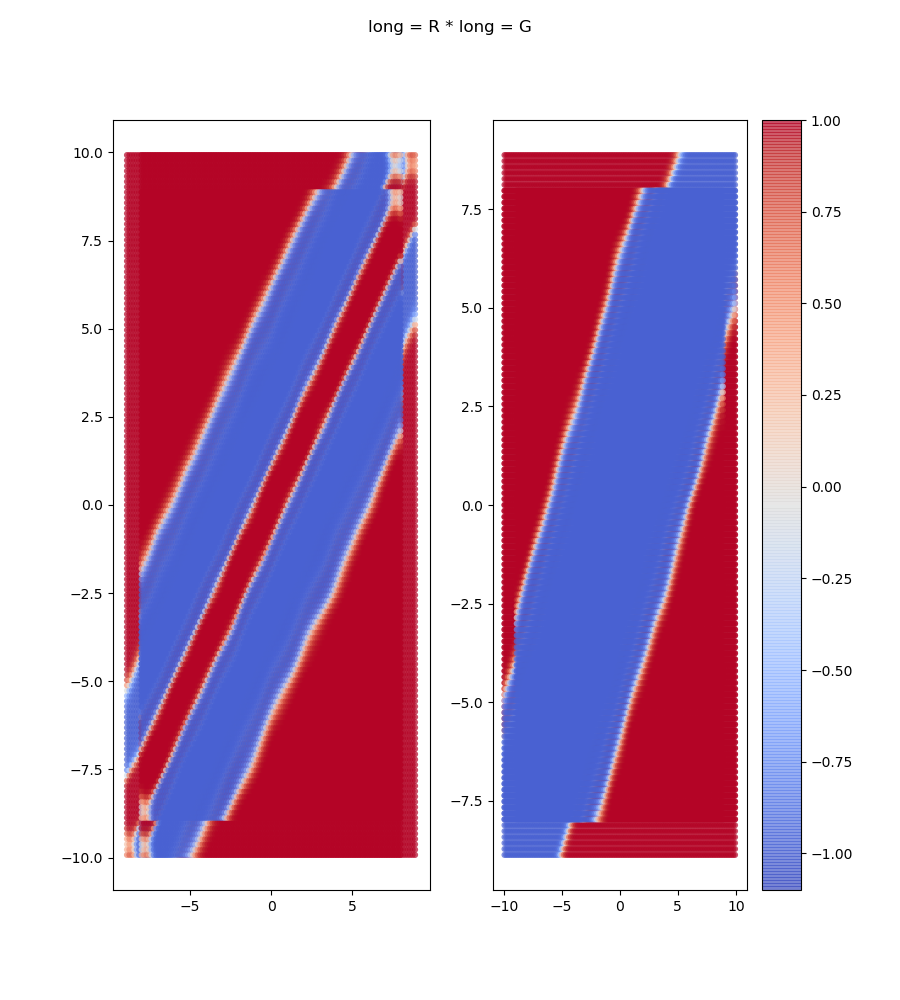}
		\caption{$HL_3$ - Cell $25$}
		\label{subf-HL3-Rawcell25}
	\end{subfigure}
	\caption{Raw activities of the exceptional pair of cells in $HL_3$}
	\label{fig-Exceptional-pair-Raw}
\end{figure}

\noindent
\begin{figure}[ht]
	\centering
	\includegraphics[width=.7\linewidth]{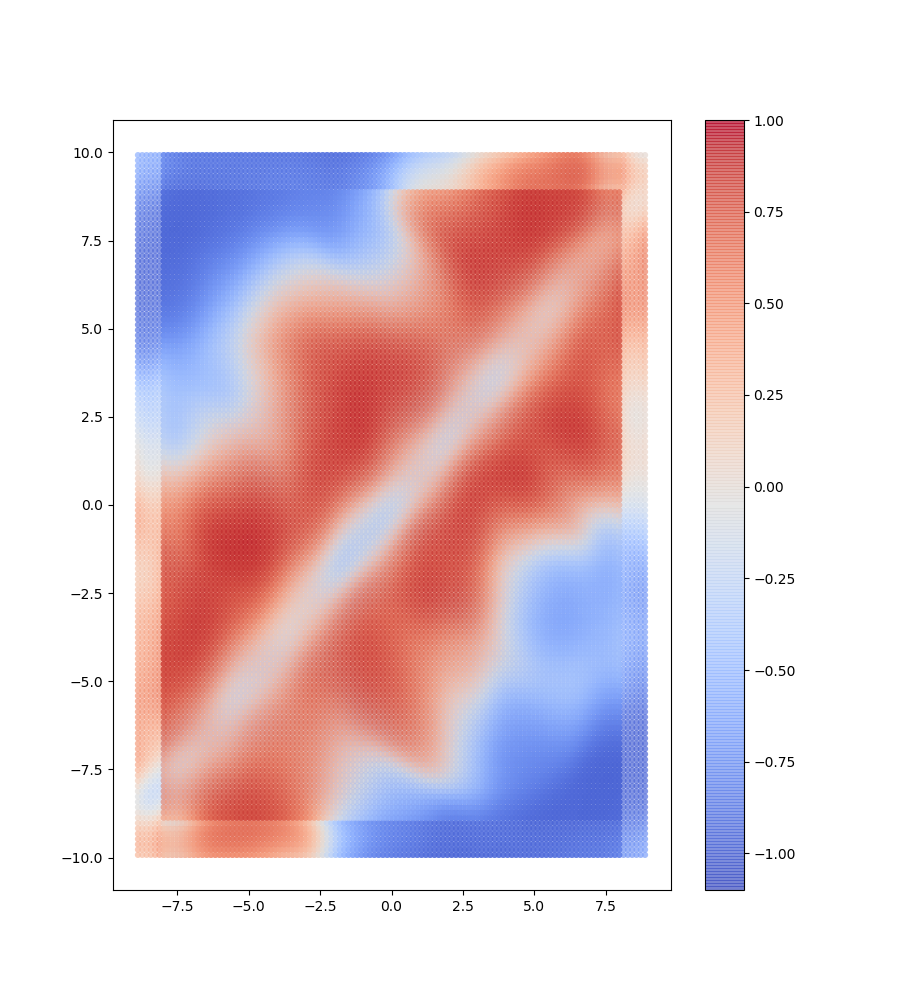}
	\caption{$HL_2$ - Cell $30$}
	\label{fig-HL2-cell30}
\end{figure}

\noindent We also analyzed two less complex problems:
\begin{enumerate}
	 \item two bars of lengths $3,5$ exchanging the colors $R$ and $G$ one time over two, during learning. This learned well, showed logic, but nothing remarkable: in $HL_2$, not too noisy, $18$ cells $D$, $18$ cells $II$, no cell $IO$, $8$ cells $II_R$, $4$ cells $II_G$, and $2$ bizarre cells. In $HL_3$, $13$ cells $D$, $7$ cells $II$, $2$ cells $II_R$, $3$ cells $II_G$, nothing for $IO$.
	 \item two bars of varying lengths $2,4,6,8$ during learning, and $3,5,7$ for testing, but the red bar being always longer than the green one. Again the result was good an logic, but without surprise in $HL_3$: in $HL_2$, $19$ cells $D$, $18$ cells $II$, $12$ cells doing a sort of Fourier analysis ($9$ symmetric, $3$ asymmetric), plus one interesting amazing cell, number $30$ telling something about $IO$, but partly localized (see Figure \ref{fig-HL2-cell30}):  $-1$ tells $\neg II$, $+1$ tells $\neg IO$, however the message is noisy.\\
	 In $HL_3$, $15$ cells $D$ and $10$ cells $II$, well quantized.
\end{enumerate}

\subsection{A third experiment, the blue object incoming}

We decided to progress towards future experiments where the network could generalize in a much wider sense, not behaving like an interpolator, but showing an understanding of what a new object (or a third voice) is and analyzing its properties by itself, by analogy with the preceding questions and the generated internal cells. We think that this would probably require a change of architecture, but we started studying what kind of cells may appear in our simple architecture in presence of a new object.\\

Two objects are presented within the line of length $23$, a red one of length $6$ and a green one of length $5$. However, in half of the cases, a blue object of length $3$  appears. It is detected by $55$ cells reacting to the blue color somewhere.\\
Again three layers, with six questions at output: $D$, $II$ and $IO$ regarding $R$ and $G$, conditioned by $B$ or $\neg B$. We note $G=S$ (for small) and $R=L$ (for large).\\
The questions asked inside are as before, except they are doubled by the conditioning $B$ or $\neg B$.\\
The results are excellent, both for the quality of learning and the very few errors made.\\

\noindent
\begin{figure}[ht]
	\begin{subfigure}{.48\textwidth}
		\centering
		\includegraphics[width=.95\linewidth]{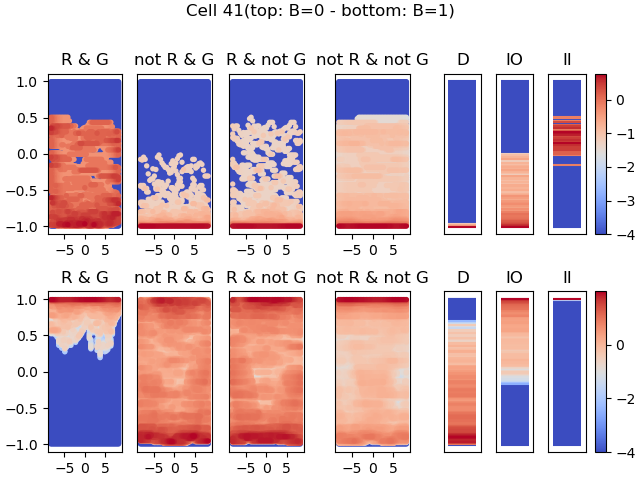}
		\caption{$HL_2$ - Cell $41$}
		\label{subf-HL2-cell41}
	\end{subfigure}
	\hfill
	\begin{subfigure}{.48\textwidth}
		\centering
		\includegraphics[width=.95\linewidth]{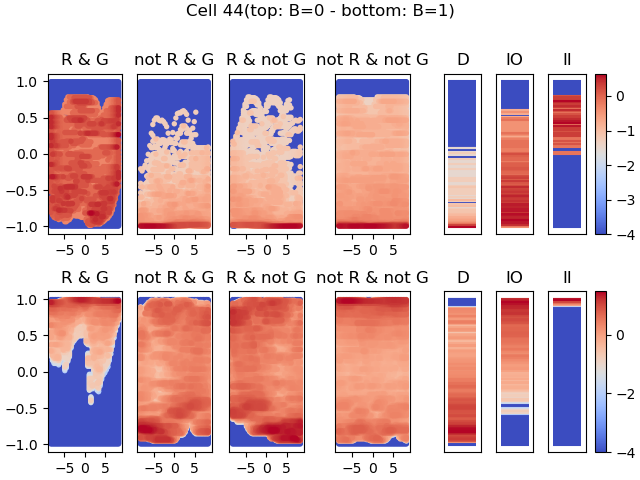}
		\caption{$HL_2$ - Cell $44$}
		\label{subf-HL2-cell44}
	\end{subfigure}
	\caption{Log activities of the exceptional cells in $HL_2$}
	\label{fig-pair-41-44-blue}
\end{figure}

\noindent
\begin{figure}[ht]
	\begin{subfigure}{.48\textwidth}
		\centering
		\includegraphics[width=.95\linewidth]{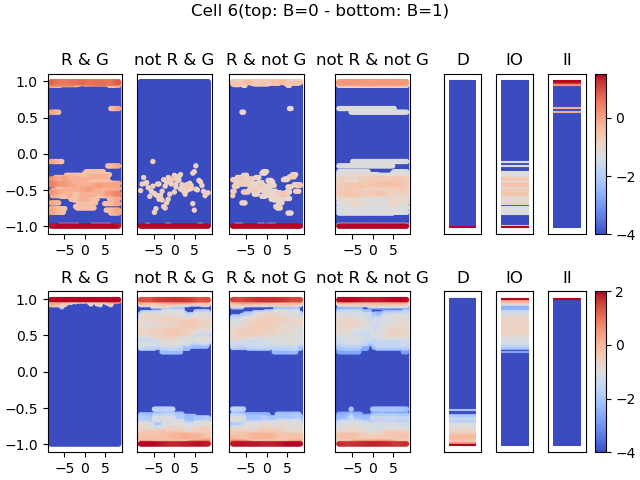}
		\caption{$HL_3$ - Cell $6$}
		\label{subf-HL3-cell6}
	\end{subfigure}
	\hfill
	\begin{subfigure}{.48\textwidth}
		\centering
		\includegraphics[width=.95\linewidth]{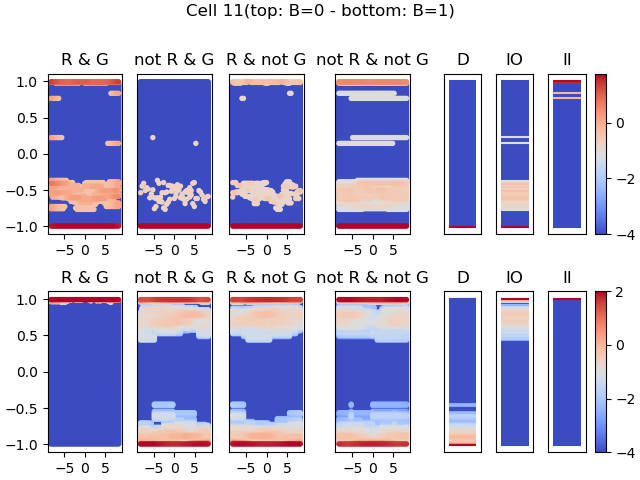}
		\caption{$HL_3$ - Cell $11$}
		\label{subf-HL3-cell11}
	\end{subfigure}
	\caption{Log activities of saturated cells in $HL_3$}
	\label{fig-pair-6-11-blue}
\end{figure}

\noindent In $HL_2$, almost all cells saturate at two values and all of these ones express logical proposition in the Boolean algebra of the objectives,
the most numerous population reacts at the presence or absence of $B$ (they are $15$, then very redundant), after that comes the population of $D$ cells conditioned by $B$ (they are $9$ and $D$ cells without conditioning ($7$ cells, with the same activity for $B$ and $\neg B$, they could have exchanged, but doing that they were not in the "final algebra"!), then the $D$ cells conditioned by $\neg B$ ($6$ cells), and the $II$ cells without conditioning (also $6$ cells, same sign for $B$ and $\neg B$),
then only one $II$ cell conditioned by $B$ (why so few? it is mysterious), two apparently uninformative cells saturated at one value, and two exceptional cell, number $41$ and $44$ (see Figure \ref{fig-pair-41-44-blue}), which are of type $II$ for $\neg B$ and of type $D$ for $B$. Each of these cells tells $+1$ when $\neg B\wedge II$ or $B\wedge D$, and tells $-1$ when $\neg B\wedge \neg II$ or $B\wedge \neg D$. Therefore it gives the following axioms
\begin{align*}
P_+&=(\neg B,II)\vee (B,D);\\
P_-&=(\neg B,IO)\vee (\neg B,D)\vee (B,IO)\vee (B,II)=IO\vee (\neg B,D)\vee (B,II).
\end{align*}
Remarkably, two such cells with exactly the same message, but a much better saturation, appears in $HL_3$, the ones numbered $6$ and $11$ (See Figure \ref{fig-pair-6-11-blue}). The above
preferred propositions, $P_-$ being the closest to $IO$ in this experiment, also belong to the final output algebra.\\
In addition, in $HL_3$ we found, $5$ color cells (telling if $B$ is here or not), $5$ pure $D$ cells, $1$ pure $II$ cell, $4$ $D$ cells
conditioned by $B$, $2$ $D$ cells conditioned by $\neg B$, $4$ $II$ cells conditioned by $B$, $2$ $II$ cells conditioned by $\neg B$, and last
but not least the two above mentioned original cells.\\

\noindent \textbf{Exercise:} analyze the possibility of proofs starting with the axioms of the most economical groups of neurons in these populations.\\

\noindent All the cells we observed in all the experiments quantize only for unions of propositions belonging to
the objectives. Only the individual choices and the possibility of proofs they offer manifest inventions.\\
For instance, the network with the richest learning, with the four objectives $D$, $IO$, $II_R$, $II_G$ could have developed cells of type $L=R$ versus $L=G$. This proposition does not belong to the output algebra, but without any doubt, it contains useful information in order to conclude; for instance, with a $II$ cell, we immediately conclude if it is $II_R$ or $II_G$ which is true.\\
Therefore to go further, we have to consider "objectives propositions" that do not form necessarily a classification task.

\subsection{The important role of symmetries}

All periodic translations (i.e. rotations) in the circular case, plus the reflections in any pair of antipodal points form the group of isometries of the circle, which can easily be quantized and can be identified with a dihedral group $D_{N}$ (we chose $N=100$) which approximates the group of the isometries of the circle $O_2(\mathbb{R})$.
The irreducible linear representations are of dimension $2$ and can be identified with the vector space $V_n$ of linear combinations of $\cos nx$ and $\sin nx$, for $n\in \mathbb{N}$, named harmonics of degree $n$.\\
Remarkably, the first network of two layers made an pertinent Fourier analysis: in order to discriminate the green from the blue, it possesses six cells
of degree $n=1$. Their phases follow a uniform distribution. Moreover, most of the blue cells correspond to second degree harmonics, and several cells
to degree $3$ (See Figure \ref{fig-Fourier}). 

\noindent
\begin{figure}[H]
	\begin{subfigure}{.48\textwidth}
		\centering
		\includegraphics[width=.75\linewidth]{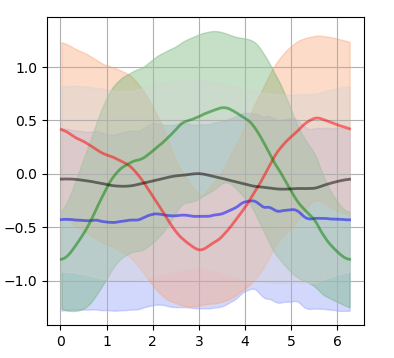}
		\caption{$HL_2$ - Cell $4$}
		\label{subf-Fourier-cell4}
	\end{subfigure}
	\hfill
	\begin{subfigure}{.48\textwidth}
		\centering
		\includegraphics[width=.75\linewidth]{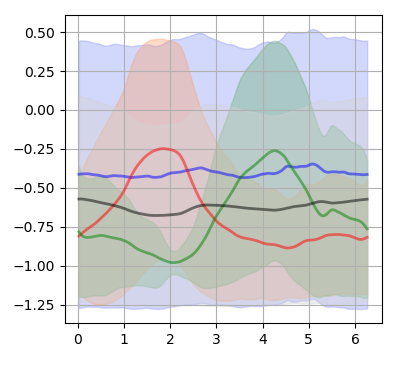}
		\caption{$HL_2$ - Cell $5$}
		\label{subf-HL2-cell5}
	\end{subfigure}
	\begin{subfigure}{.48\textwidth}
		\centering
		\includegraphics[width=.75\linewidth]{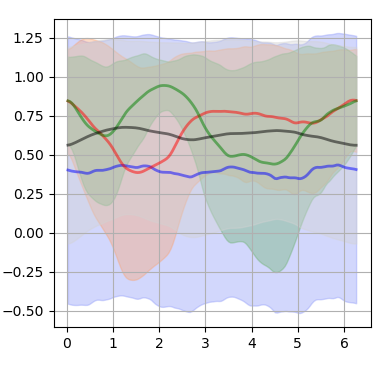}
		\caption{$HL_2$ - Cell $7$}
		\label{subf-Fourier-cell7}
		\end{subfigure}
	\hfill
	\begin{subfigure}{.48\textwidth}
		\centering			\includegraphics[width=.75\linewidth]{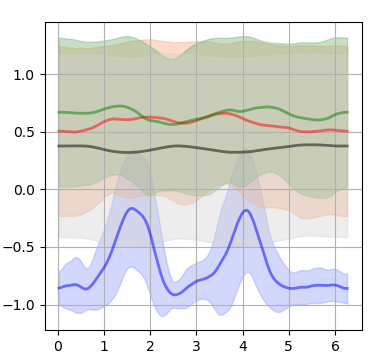}
		\caption{$HL_2$ - Cell $27$}
		\label{subf-Fourier-cell27}
	\end{subfigure}
	\begin{subfigure}{.48\textwidth}
		\centering
		\includegraphics[width=.75\linewidth]{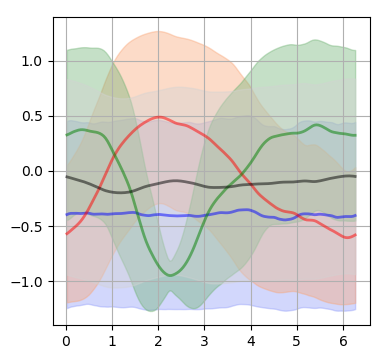}
		\caption{$HL_2$ - Cell $35$}
		\label{subf-Fourier-cell35}
	\end{subfigure}
	\hfill
	\begin{subfigure}{.48\textwidth}
		\centering
		\includegraphics[width=.75\linewidth]{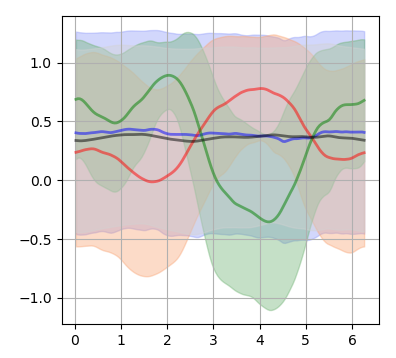}
		\caption{$HL_2$ - Cell $47$}
		\label{subf-Fourier-cell47}
	\end{subfigure}
	\caption{$HL_2$ - $R\wedge G$: Blue; $R\wedge\neg G$: Red; $\neg R \wedge G$: Green; $\neg R \wedge\neg G$: Black}
	\label{fig-Fourier}
\end{figure}

Moreover many curves are symmetric with respect to an axis, expressing an interest for the reflections.\\
With three hidden layers, this spontaneous use of Fourier analysis becomes less evident in $HL_2$, and is replaced by invariants under $D_{100}$ in $HL_3$.\\

In the linear case, where $D$ is a finite segment, there is only one
non-trivial symmetry, the reflection with respect to the middle, or mirror symmetry. This implies that, with one hidden layer, most of the propositional curves are anti-symmetric with respect to zero, the center of the segment (See Figure \ref{fig-Fourier-linear}). 

\noindent
\begin{figure}[H]
	\begin{subfigure}{.48\textwidth}
		\centering
		\includegraphics[width=.75\linewidth]{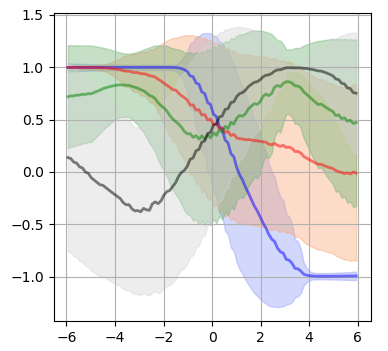}
		\caption{$HL_2$ - Cell $3$ - Linear}
		\label{subf-Fourier-cell3-linear}
	\end{subfigure}
	\hfill
	\begin{subfigure}{.48\textwidth}
		\centering
		\includegraphics[width=.75\linewidth]{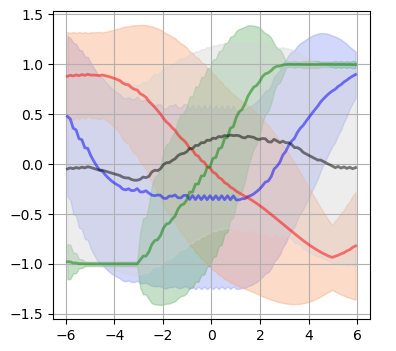}
		\caption{$HL_2$ - Cell $4$ - Linear}
		\label{subf-HL2-cell4-linear}
	\end{subfigure}
	\begin{subfigure}{.48\textwidth}
		\centering
		\includegraphics[width=.75\linewidth]{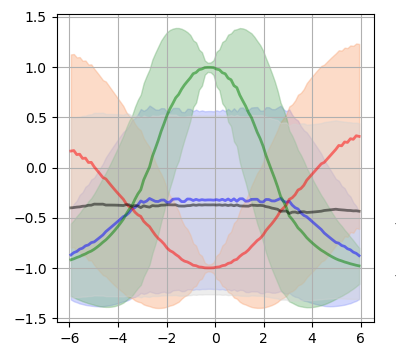}
		\caption{$HL_2$ - Cell $5$ - Linear}
		\label{subf-Fourier-cell5-linear}
	\end{subfigure}
	\hfill
	\begin{subfigure}{.48\textwidth}
		\centering			\includegraphics[width=.75\linewidth]{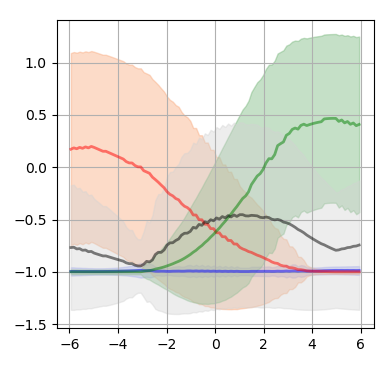}
		\caption{$HL_2$ - Cell $11$ - Linear}
		\label{subf-Fourier-cell11-linear}
	\end{subfigure}
	\begin{subfigure}{.48\textwidth}
		\centering
		\includegraphics[width=.75\linewidth]{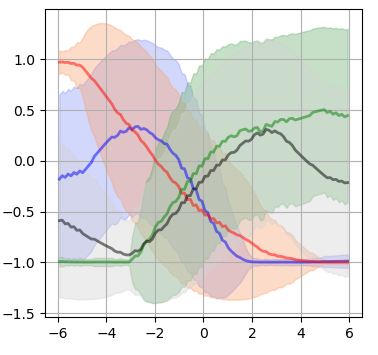}
		\caption{$HL_2$ - Cell $22$ - Linear}
		\label{subf-Fourier-cell22-linear}
	\end{subfigure}
	\hfill
	\begin{subfigure}{.48\textwidth}
		\centering
		\includegraphics[width=.75\linewidth]{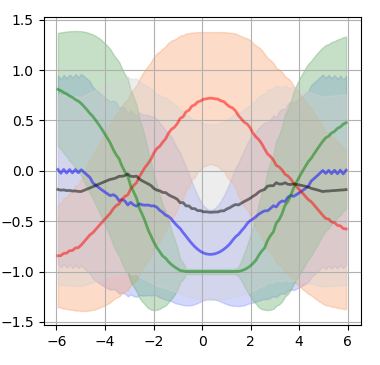}
		\caption{$HL_2$ - Cell $31$ - Linear}
		\label{subf-Fourier-cell31-linear}
	\end{subfigure}
	\caption{$HL_2$ - $R\wedge G$: Blue; $R\wedge\neg G$: Red; $\neg R \wedge G$: Green; $\neg R \wedge\neg G$: Black}
	\label{fig-Fourier-linear}
\end{figure}

For two hidden layers, in $HL_2$, most of these curves are symmetric with respect to the vertical axis. This corresponds to the two simplest representations of the cyclic group $C_2$. With three layers all the cells, with two exceptions, are predicative, with almost all curves horizontal, thus in the trivial representation of $C_2$. \\

With only two layers, the network was a little bit more successful in the linear case; maybe in this case, Fourier analysis on the circle is more disturbing than
helpful.\\

Then we notice that the network takes great care of the symmetries in the data. Moreover it is remarkable that by itself with so few cells,
it does a generalized Fourier analysis, using the simplest irreducible linear representations of the group of symmetries.\\

Also, remind that in the first experiment, the simple fact that the metric at the end was changed in order to respect the symmetries in the logic
of the questions asked, was determinant for the success of the network.\\

\paragraph{Toposic interpretation (not necessary to read)}
In toposic terms, this corresponds to the importance of the groupoid attached to each layer. This groupoid takes in account the symmetries in the data and the semantic of the questions. The logic in the topos of pre-sheaves over a groupoid is Boolean. This agrees with the Boolean character of the semantic in each individual layer. Then the apparition of generalized Fourier analysis in this context, is not surprising from the point of view of topos, because the linear
representations (over any commutative ring) are nothing else that the Abelian objects in the topos in fiber of each layer. The proper logic of the topos of the site
of the network gives the opposite: a non-Boolean progress in the logic when going deeper in the layers.\\
\noindent The groupoid $G_3$ of the first model was made of three objects $a,b,c$, non related, and one circle of morphisms for each object. Each layer corresponds
to a representation of $G_3$. The logic is made by the vertices, edges and face of the triangle $\Omega_3$, plus the empty set.\\
The external symmetries of $G_3$ itself are described by the group of permutations $\mathfrak{S}_3$.\\

\begin{rmk*}
	\normalfont According to the questions asked in our second series of experiments, we even could have expected the intervention of
	a sort of quantization of the group of all homeomorphisms of the space (circular or linear); perhaps when augmenting the complexity of the data
	and the number of neurons, this would appear.
\end{rmk*}

\subsection{The important role of redundancy}

Most  cells in the last hidden layer are logical and, before the intervention of the blue object, they are repeated several times almost identically.
Even with the blue one, the tendency to repeat is obvious. As this does not appear in the preceding layers, this shows the existence of a standard way of giving less weight to less logical cells in order to produce logical behaviors.\\

There is an analogy with statistics: the advantage to use independent identically distributed variables. But also this can be linked to learning,
because the minimum of the metric can have highly specific properties.\\

In this second set of experiments, we saw that the probabilities are necessary if we want to conclude with logical arguments from the last hidden layer.
For instance, each time a cell saturates at $+1$ for the proposition
$R(a)\wedge G(a)$, if its activity for a given image $\xi$ is not $+1$ (sufficiently clearly) we directly conclude $D$, but if it is $+1$ we cannot,
except if we add information coming from the other curve; if the distributions of activity associated to the three other local propositions do not clearly saturate at $+1$, we are allowed (and probably also the network) to deduce from the activity $+1$ (or almost $+1$), that there exists a point $a$ where $G(a)$ and $V(a)$ is true. Consequently, in a statistical sense, if many cells repeat this message it has more chance to be true, and we get a proposition not far from a certitude. This will be an important ingredient of a discussion about semantic information.\\

There remains an important experimental question: is it visible in the weights chosen by learning that the network uses the logical deduction? And how the statistical argument
is taken into account?\\
The existence of the logical cells and the evidence of a role of probabilistic inference give no great doubt that both logic and statistics are used, but it clearly needs to be made more precise.\\

\section{How do weights perform logic and deduction?}

First, we developed an automatic detection and analysis of the logical behaviors of cells, adaptable to all the above reported
experiments. \\
\indent First step has been to record the reactions of individual cells $a\in L_k;k=2,3$, to the input data $\xi$, conditioned by a 
known answer $\sigma$ to one of the questions that were asked inside,
and frequently reflected the objectives of the network. For instance, we recorded the distribution activity of $a$ for the inputs where we know that the green object is included in the red
one.\\
Then its activity, a real number $x_a(\xi|\sigma)$ between $-1$ and $1$, is partitioned in three sets labelled by $-1$, $0$ and $+1$,
according to the values smaller than $-1/3$, between $-1/3$ and $1/3$ and greater than $+1/3$. The cell is considered to be \emph{logical} for the condition $\sigma$
if more than $80 \%$ of the conditioned activities belong to one of the interval of the partition. It appeared in all the above experiments that when a cell is
logical, the chosen segments contain an extremity, thus attributing the number $-1$ or $+1$ to $a$ for the condition $\sigma$. (Remark we have now, in more complex experiments, cells which also choose the interior $]-1/3,1/3[$ under some condition.) Remark that, in most cases with three hidden layers or more, most cells are logical
for at least two conditions, and have distributions much more concentrated than $80 \%$, around $95 \%$ (see table \ref{fig:distri}).\\
\begin{table}
	\centering
	\includegraphics[width=.45\linewidth]{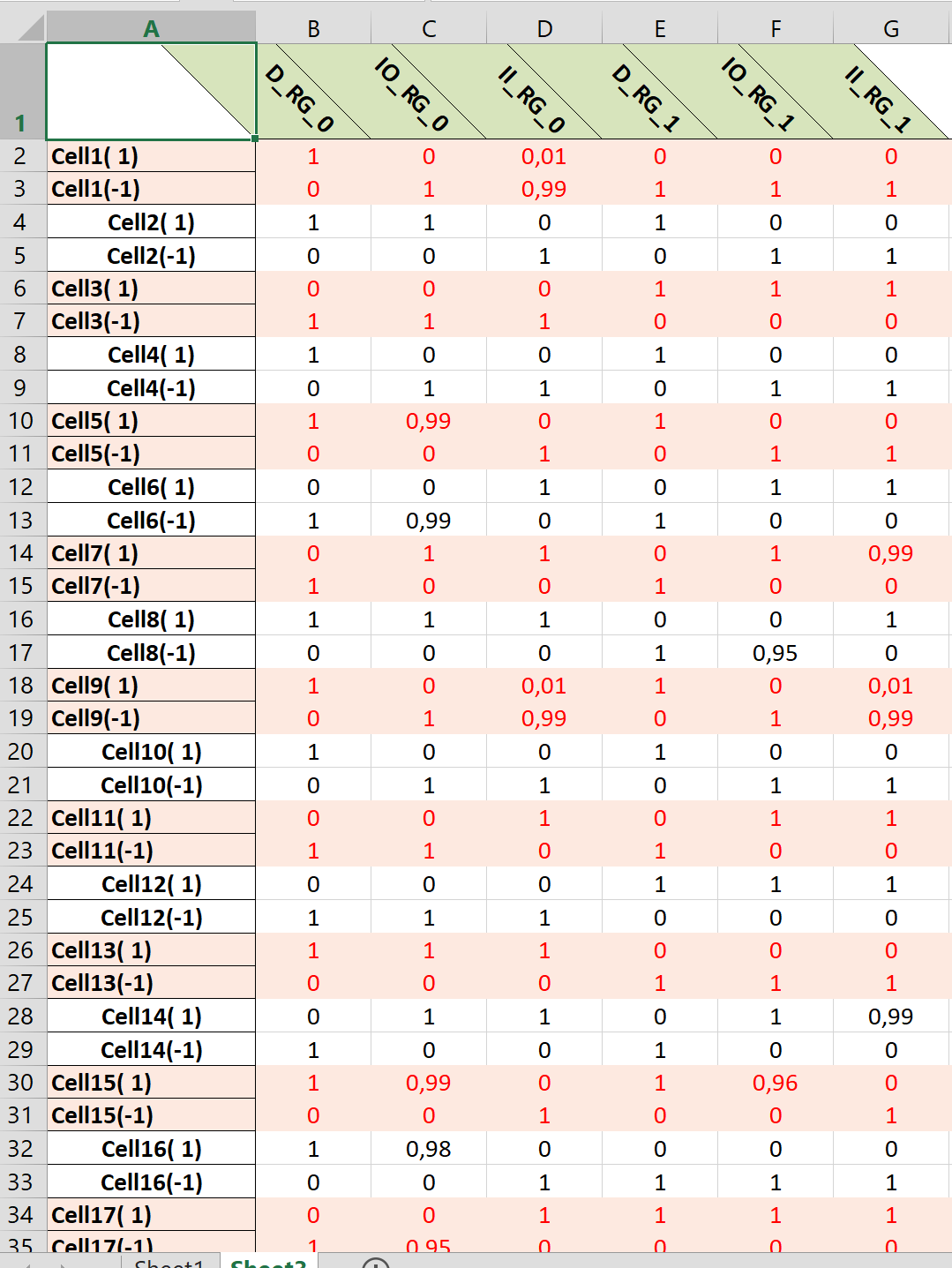}
	\caption{Logical cells behavior}
	\label{fig:distri}
\end{table}

Next step is to decode the activity: when the logical cell fires in the opposite interval $\mp 1$ of its preferred interval for $\sigma$, 
we consider that it excludes $\sigma$, and we put a $0$ in the table
for this proposition $\sigma$, meaning $\neg\sigma$ is asserted by $a$ at this value $\mp 1$. If nothing can be decided, we put $1$. This gives a matrix, whose columns represent the
conditions and its lines the pairs $(a,\varepsilon)$, where $a$ is a cell and $\varepsilon$ is $+1$ or $-1$. The logical score is defined by the number of
$0$ in the line of $a,\varepsilon$.\\
\indent Starting with these data, we computed the predictions of the pairs $(a,b)$ and of the triples $(a,b,c)$ of cells, at given values
$\varepsilon_a,\varepsilon_b,...$ in $\{\pm 1\}$. This gave matrices which four, resp. eight, lines for each pair, resp. triple. The core of
the matrix in general possesses less rows, because not all vector of signs can be realized given the possible inputs. For instance, if the
objectives are $D_{RG}$, $II_{RG}$, $IO_{RG}$, at most three vectors are accessible.\\
It appeared that no pair can reconstruct (except one exception) all the conditions, but a non-negligible subset of triples can do (around $5 \%$
of the possible triples). To each triple, we gave a score $N(a,b,c)$, which is the number of conditions that it can reconstruct. The efficient
ones are named here \emph{conclusive}, cf. section \ref{sec-predicate}. For instance in the case of three colors, $R,G,B$ where $B$ is sometimes present sometimes not, and the
objectives are $D_{RG}B$, $II_{RG}B$, $IO_{RG}B$, $D_{RG}\neg B$, $II_{RG}\neg B$, $IO_{RG}\neg B$, the conclusive triples have a score of $6$.
The core matrix $A$ of a triple has six lines and six columns.\\

On the weights side, we computed, when the network has learned, the ones corresponding to the connection between the last hidden layer and the neurons
of the output. They are real numbers of any possible signs. For instance, in the above experiment with three colors, each of the $25$ 
hidden neuron $a$ in $HL_3$ defines a column vector of $\mathbb{R}^{6}$.\\
\indent For a given set $D$ of $d$ neurons in $HL_3$ we get a matrix $W(D)$ with $6$ rows and $d$ columns. Remind that between the last
hidden layer and the output the transformation is linear. Thus we can define a \emph{quantized expression} of $D$, which associates to any
quantized activity vector of the $d$ neurons an answer in the numerical output layer $\mathbb{R}^{6}$.\\
In the particular case of the triples that we described in the preceding paragraph, we have for each triple, a product of $6\times 6$ matrices
\begin{equation}
M(a,b,c)=W(a,b,c).A(a,b,c),
\end{equation}
which describes the estimation of a condition $\sigma_{out}$ made by the triple from a condition $\sigma_{in}$.\\

\noindent The hypothesis of a logical functioning is that this matrix is closed to the identity, then almost diagonal with diagonal values close to $1$.
Therefore we define a \emph{weighted logical score} of the triple by the formula
\begin{equation}
\mu_W(a,b,c)=\| \diag (M)\|_{\ell_1}-\frac{1}{6}\| M\|_{\ell_1};
\end{equation}
the $\ell_1-$norm being the sum of the absolute values of the coefficients.\\

\noindent Beside this score, we can take as \emph{brut weight score} the norm $\| W\|_{\ell_1}$ of $W(a,b,c)$.\\
\noindent
\begin{figure}[htb!]
	\begin{subfigure}{.48\textwidth}
		\centering
		\includegraphics[width=.95\linewidth]{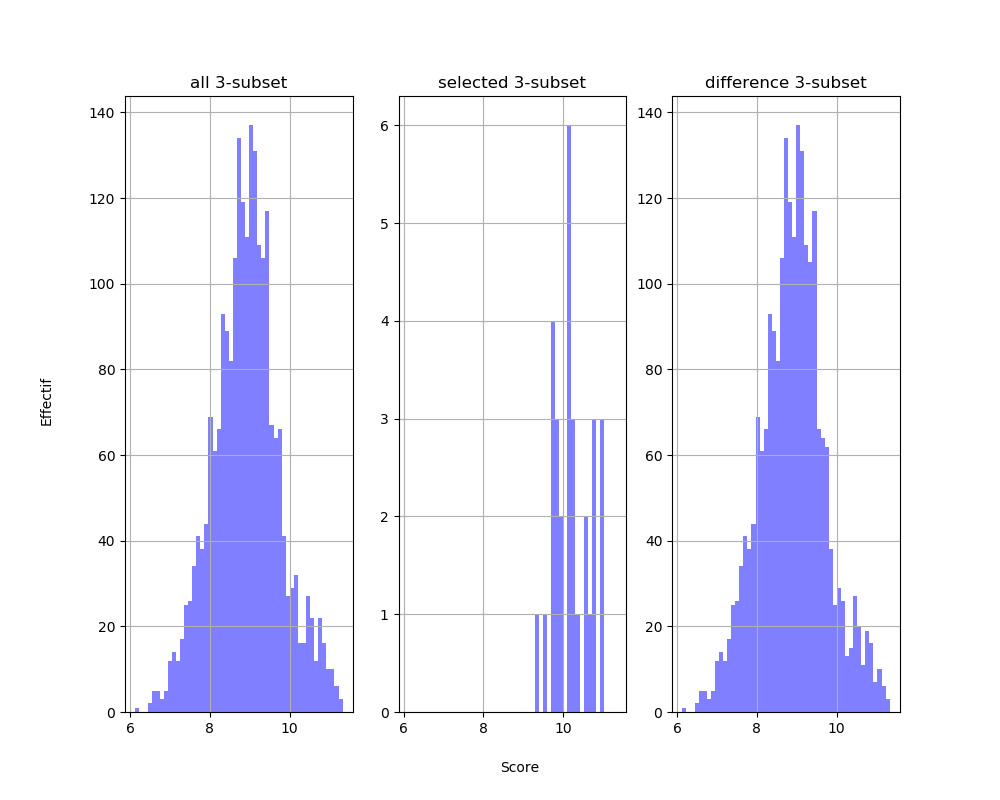}
		\caption{Raw $\ell_1$ norm of the weight matrices}
		\label{subf:weights}
	\end{subfigure}
	\hfill
	\begin{subfigure}{.48\textwidth}
		\centering
		\includegraphics[width=.95\linewidth]{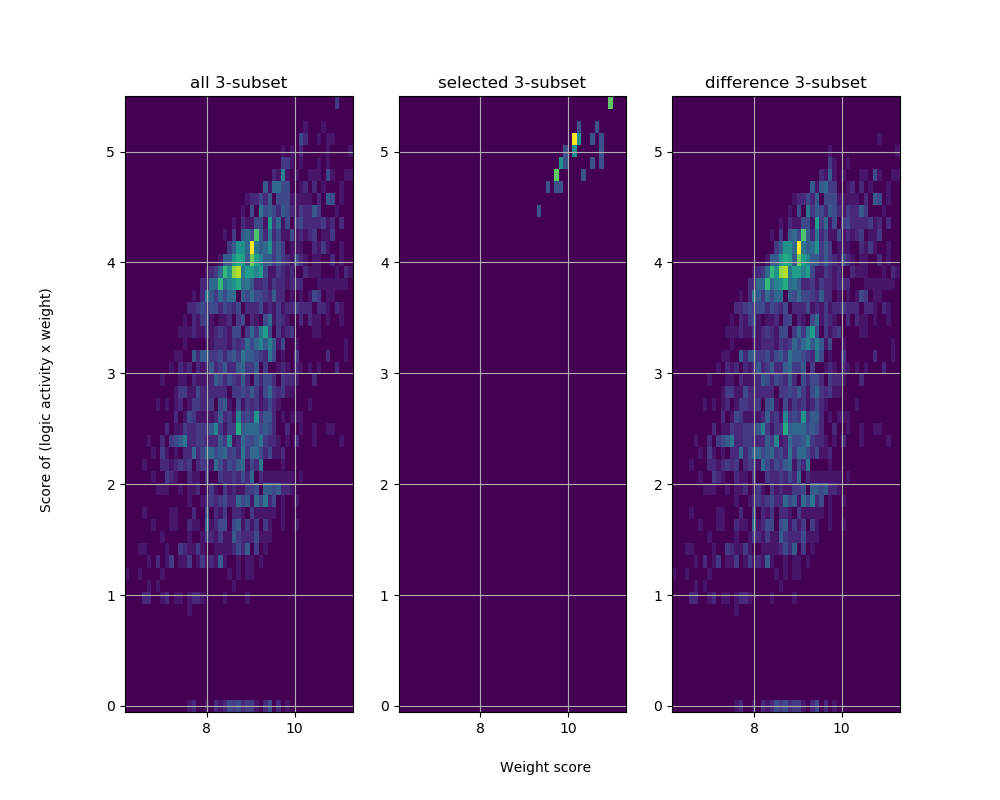}
		\caption{Correlation between $\ell_1$ norm and logical values}
		\label{subf:logicLW}
	\end{subfigure}
	\begin{subfigure}{.48\textwidth}
		\centering
		\includegraphics[width=.95\linewidth]{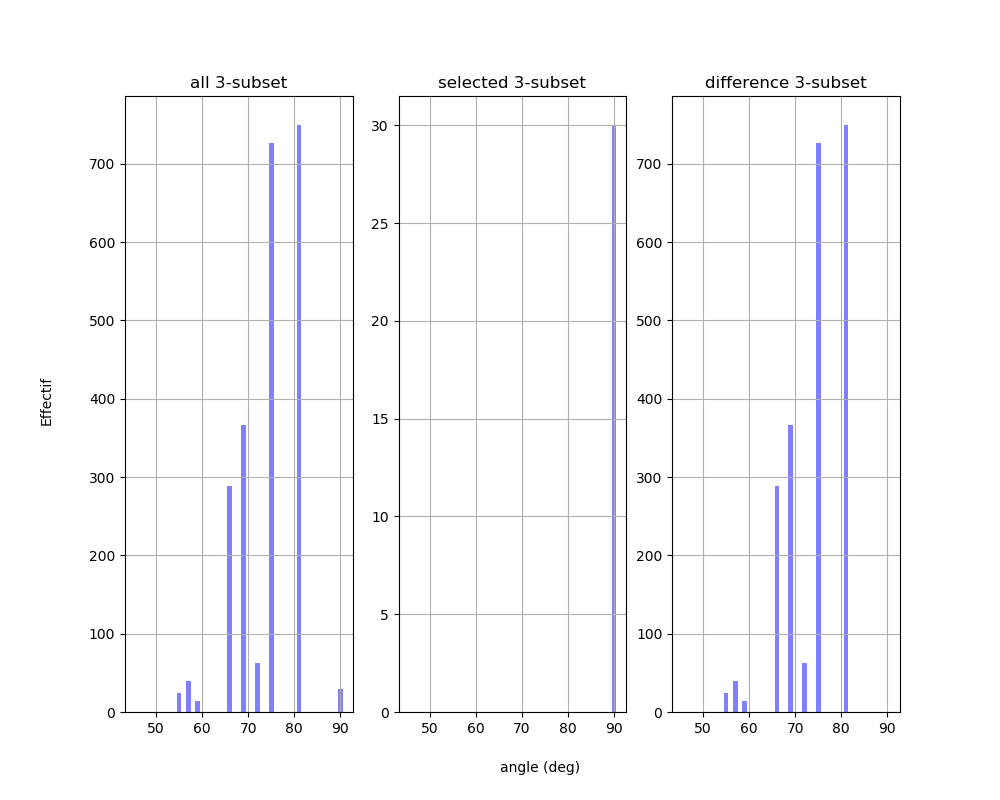}
		\caption{Histogram of deductions}
		\label{subf:logicL}
	\end{subfigure}
	\hfill
	\begin{subfigure}{.48\textwidth}
		\centering
		\includegraphics[width=.95\linewidth]{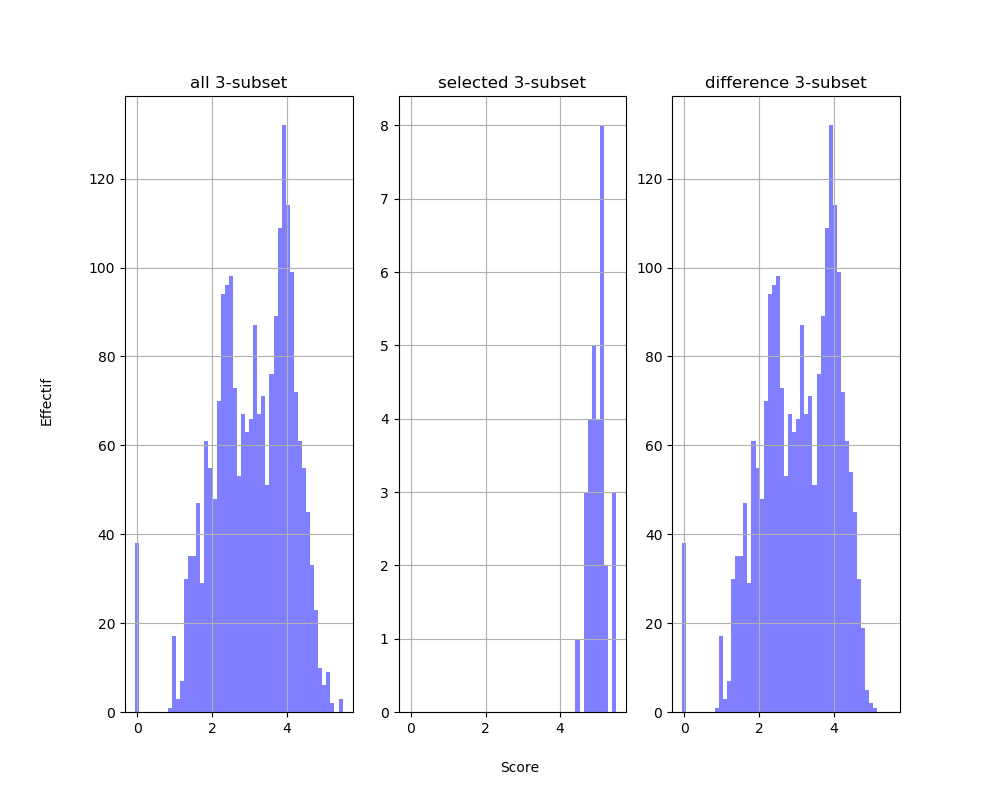}
		\caption{Histogram of deductive power}
		\label{subf:logicxW}
	\end{subfigure}
	\begin{subfigure}{.48\textwidth}
		\centering			\includegraphics[width=.95\linewidth]{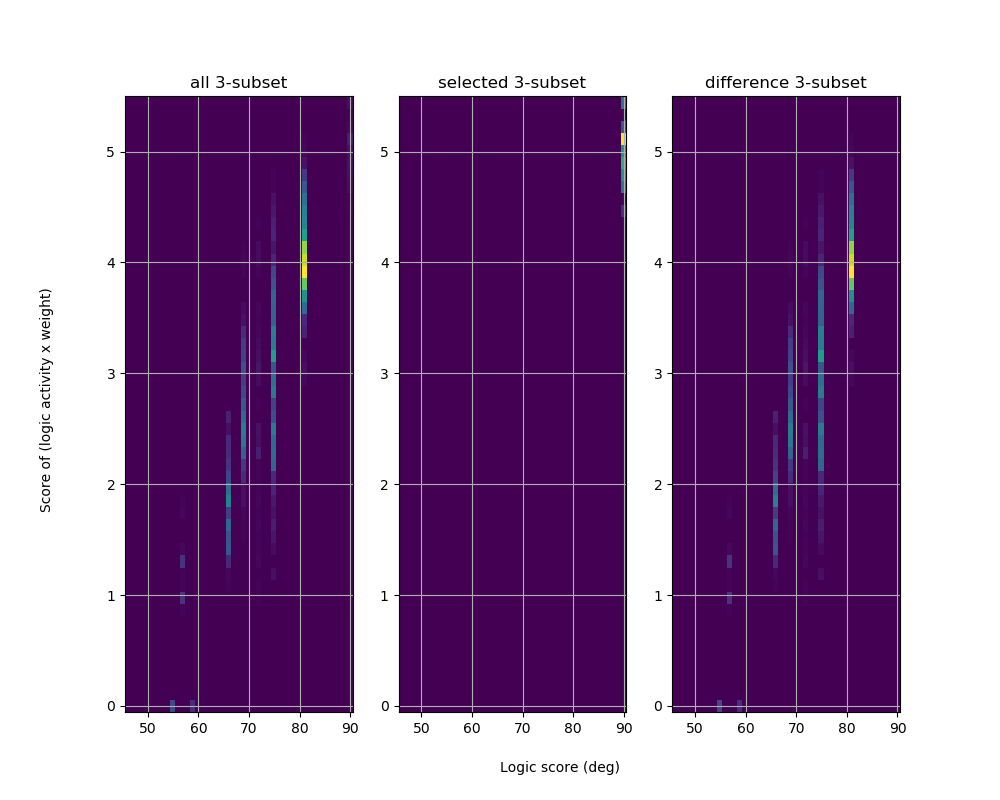}
		\caption{Correlation between the logic score and the weight score}
		\label{subf:logicxWL}
	\end{subfigure}
	\caption{Triple of conclusive cells and others}
	\label{fig:weights}
\end{figure}

\noindent The most important statistical result is the positive correlation
between the pure logical score $N(a,b,c)$ of a triple and irs weight's scores, either the logical one either the brut one.\\

\noindent The analysis of the weights distributions, according to $\mu_W$ or to $\| W\|$ clearly shows that the logical deduction
is reflected in the weights. Cf. figure \ref{fig:weights}.\\

\noindent To get a good linear regression in both comparisons, we needed to take a convex function of $N$, computed as
an angle, associated to the scalar product of the matrix $A$ withe the matrix $\mathbf{1}_A$. This indicates that the passage from logic to weight is concave like.\\

The statistics are on the triple of cells, the selected triples are the conclusive ones that resolve completely the $6$ objectives. Figure \ref{subf:weights} shows the raw $L1$ norm of the weight's matrices $3\times 6$.
Figure \ref{subf:logicLW} shows the correlation between this norm and the logical value of the triples, i.e. what the three cells prove together.\\
Figure \ref{subf:logicL} is the histogram of deductions (counted as an angle).\\
Figure \ref{subf:logicxW} is the histogram of deductive power of the weights applied to the quantized activities.\\
Eventually, figure \ref{subf:logicxWL} is the correlation between the logic score of the triple (counted as an angle) and the preceding score of the weight
proofs.

%

\section{Conclusion and discussion: Questions and Perspectives}
Our main result is the observation of spontaneous development of logical
activity in simple DNNs. The corresponding information structure is not only
statistical, even if statistics play an important role, because it relies on logical deduction and inference, and it is related to semantic as it is understood
usually in linguistic or meta-linguistic. As discussed in the text, other characteristics of this structure are redundancy and symmetry. Another discovery,
which was not explicit in the conjectured Information Bottleneck principle
or in the Infomax principle, is the fundamental role of individual cells: even
if their characteristics heavily depend on the layer where they are embedded,
each neuron develops its own personality, for contributing to the collective
answer of the layer to the objectives, in function of the stimuli.\\

\noindent The tasks in the experiments were of the type of classification problems;
the main invention of the network consisted in the introduction of the full
Boolean algebra over the elements of the classification. Proofs are supported
by the weights, and correspond to a spontaneous modularity. Moreover, the
elegance of the proofs follows a remarkable progress, when adding layers and
adapting the metric.

\subsubsection*{Perspectives}	
\begin{enumerate}[label=\arabic*)]
	\item Then the natural next step was to obtain more inner propositions and
	theories, than union of the elementary objectives. We have obtained partial
	results in this direction, that we will present soon in Logical Information
	Cells II \cite{logic-DNN-2}. 
	
	In these new experiments, we played with the same kind of data (colored bars, with three colors) and similar problems (about global topology), by changing the architectures, going from chains to recurrent neural networks RNN, then to graphs of interacting RNNs. The networks developed completely new types of cells, inaccessible by adding layers en layers, and allowing to address new kinds of problems.
	
	An entirely new phenomenon appears: the network develops by itself
	propositions and theories that do not belong to the Boolean algebra generated by the objectives. This answers positively (however weakly) the question of the invention of logic.
	\item One of the most important challenge with artificial neural networks is
	to obtain understandable generalization out of the learning data. In the
	examples we constructed, the main ingredient for obtaining results in this
	direction is the change of the architecture, other ingredients being more
	complex task and larger network with more cells.
	\item In these experiments, logical cells were still present. This gives a positive
	answer to the scaling problem: can logical behaviors of cells resist to
	the enlargement of DNNs?
	\item {\bf Problem:} develop further the relation between the examples of generalization due to changes of architectures and the theoretical arguments based
	on invariance structures, as in \cite{belfiore2021topos}.
	\item From the semantic point of view, the main problem is : how to make
	the cells able to use (at least implicitly) a sufficient semantic, i.e. types and
	contexts, objects and properties, dependent judgments of type and truth.
	One of the difficulties is to find methods rendering evident the use of abstract reasoning and semantic activity, beyond the combinatorics (that is not
	nothing, but far from thinking). The problem is similar to the problem posed by the research of reasoning in animals other than humans, even without
	considering consciousness, or causality. The next problem is to construct net-
	works and learning methods able to transmit this kind of abstract knowledge
	to another network.
	\item According to the suggestions of our theoretical paper \cite[section 3.2]{belfiore2021topos}, we also have to explore the gain we may expect from spontaneous activity.
\end{enumerate}
\section*{Acknowledgments}
The authors wish to warmly thank Merouane Debbah for his deep interest, the help and the support he gave, and Zhenrong Liu (Louise) for her constant and very kind help at work. They also warmly thank Ingmar Land, Enrique Yamamoto and Apostolos Destounis for their very stimulating interest, and for having confirmed
on their own by doing a lot of other original experiments, the preliminary results of the present paper.
Special thanks are due to Alain Berthoz, who attracted the interest of DB on the articles of Neromyliotis and Moschovakis \cite{nero-moscho-17,nero-moscho-18}, that was the starting point of this whole line of research.
\vspace{1cm}
\bibliographystyle{alpha}
\bibliography{bib/semantic}

\end{document}